\documentclass{article} 
\usepackage{iclr2027_conference,times}

\usepackage{amsmath,amsfonts,bm}

\def\eqref#1{equation~\ref{#1}}

\def\1{\bm{1}}

\DeclareMathAlphabet{\mathsfit}{\encodingdefault}{\sfdefault}{m}{sl}
\SetMathAlphabet{\mathsfit}{bold}{\encodingdefault}{\sfdefault}{bx}{n}

\usepackage{microtype}
\usepackage{amsmath, amssymb, amsthm}
\usepackage{mathtools}
\usepackage{bm}
\usepackage{booktabs}
\usepackage{multirow}
\usepackage{array}
\usepackage{tabularx}
\usepackage{graphicx}
\usepackage[table]{xcolor}   
\usepackage{algorithm}
\usepackage{algpseudocode}
\usepackage{enumitem}
\usepackage{caption}
\usepackage{subcaption}
\usepackage{url}
\usepackage{textcomp}
\usepackage{arydshln}
\usepackage{cite}

\usepackage{wrapfig}   

\usepackage{xcolor}
\usepackage{listings}

\lstdefinestyle{promptstyle}{
    backgroundcolor=\color{gray!10},    
    basicstyle=\ttfamily\small,
    breaklines=true,
    breakatwhitespace=true,
    showspaces=false,
    showstringspaces=false,
    showtabs=false,
    tabsize=4,
    frame=none, 
}

\usepackage[table]{xcolor}
\definecolor{slateHeader}{HTML}{e2e8f0}  
\definecolor{slateMid}{HTML}{f1f5f9}     
\definecolor{slateLite}{HTML}{f8fafc}    
\definecolor{oursBlue}{HTML}{eff6ff}     

\usepackage{dsfont}
\newcommand{\ind}[1]{\mathds{1}}

\usepackage{hyperref} 
\usepackage[capitalise]{cleveref}
\theoremstyle{remark}
\newtheorem{assumption}{Assumption}

\newtheorem{lemma}{Lemma}
\newtheorem*{lemma*}{Lemma}

\title{ProV-Mem: Proactive Memory for Streaming Video Question Answering}

\iclrfinalcopy

\author{
Bingbing Wang$^{1}$, ~
Jing Li$^{1,2}$\thanks{\quad Corresponding Author} ~,
Han Gao$^{3}$, ~
Yishi Lin$^{3}$, ~
Ruifeng Xu$^{4}$, ~
Peng He$^{3}$\\
    $^{1}$ Department of Computing, The Hong Kong Polytechnic University, Hong Kong, China \\
    $^{2}$ Research Institute for Generative AI, The Hong Kong Polytechnic University, Hong Kong, China  \\
    $^{3}$ Tencent Inc., China\\
    $^{4}$ Shenzhen Loop Area Institute, Shenzhen, China
}

\begin{document}

\maketitle

\begin{abstract}
Proactive streaming video question answering requires a model to continuously monitor an incoming stream and autonomously decide when to speak, what to say, and whether to emit a response. Existing streaming systems face a fundamental difficulty: retaining long visual histories is costly, whereas discarding them can lead to missed events and repeated responses. We propose \textbf{ProV-Mem}, a training-free proactive memory framework built around a frozen VLM backbone. ProV-Mem decomposes the proactive decision into three lightweight memory tiers: a temporal-trigger memory controls when to query the backbone and recovers from prolonged silence, an instructional memory guides what to say through a fixed prompting prior for concise and visually grounded generation, and an episodic memory determines whether to emit by filtering near-duplicate responses. Under an idealized conditional-sufficiency assumption, we derive a stagewise approximation bound relating the decomposed process to a full-state joint
policy. We further introduce Cost-Aware Area Under the Curve (\textbf{CAUC}), which extends PAUC by explicitly discounting redundant, overly frequent, and unnecessarily long emissions. Experiments on the four scenarios of ProactiveVideoQA show that ProV-Mem achieves the best CAUC across all scenarios and the best PAUC on WEB, EGO, and TV, while maintaining low per-response inference latency.
\end{abstract}

\section{Introduction}
\label{sec:intro}
Recent advances in video-language models (VLMs) \citep{lin2024video, maaz2024video, li2024llava} have driven a paradigm shift from retrospective video understanding \citep{xu2016msr, lei2018tvqa} to real-time streaming applications, such as live moderation \citep{xu2026streamingvlm} and embodied agents \citep{driess2023palm, o2024open}.
Unlike traditional passive streaming benchmarks where explicit user queries drive every turn, Proactive Video Question Answering (Proactive-VQA) \citep{wang2025proactivevideoqa} requires an agent to continuously monitor an unbounded video stream without explicit prompts. Under strict real-time constraints, the agent must autonomously determine the timing (when to speak), content (what to say), and necessity (whether to emit) of each response, integrating trigger control directly into the task formulation.

\textbf{A dual predicament: passive memory and a leaky metric.}
To sustain real-time throughput on unbounded streams, existing architectures typically discard raw video frames, which deprives the model of explicit temporal visual history. This lack of historical awareness causes trained-head baselines to exhibit severe templated repetitions and coverage gaps, frequently missing critical events across target windows, as shown in Figure~\ref{fig:intro}(a). Latest pipelines attempt to mitigate this by relying on lightweight feature caches or hidden states to implicitly represent past context. However, these implicit representations remain fundamentally passive: they store compressed context but play no active role in governing when, what, or whether to respond, so any degradation inevitably causes all three controls to collapse simultaneously. Crucially, this systemic failure is further masked by a structural loophole in the standard evaluation protocol. The primary metric, Proactive Area Under the Curve (PAUC) \citep{wang2025proactivevideoqa}, integrates per-fire quality scores over ground-truth windows, which allows degenerate, over-firing policies to artificially inflate their scores by accumulating window duration rather than increasing information density.

\begin{figure}[!t]
  \centering
  \includegraphics[width=0.95\linewidth]{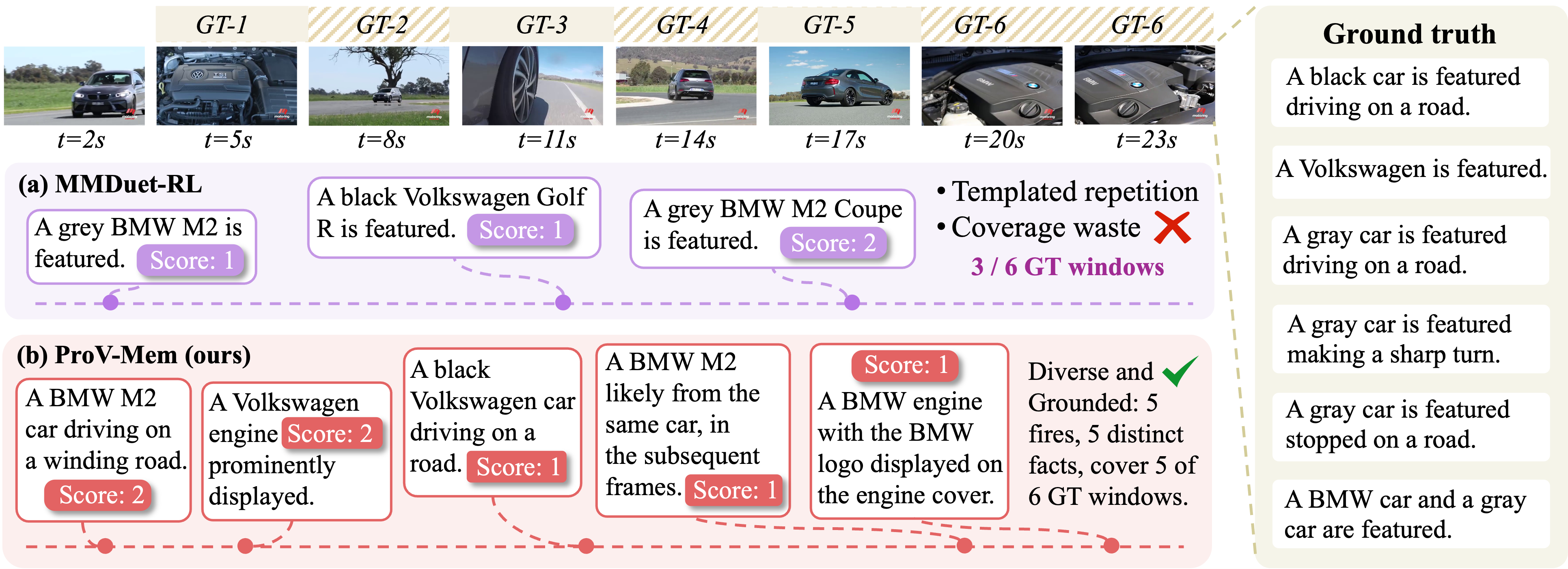}
  \caption{Comparison of (a) existing methods and (b) our ProV-Mem method.}
\vspace{-1em}
\label{fig:intro}
\end{figure}

\textbf{Our solution: from passive caching to proactive memory.}
The root cause of implicit context decay is that a passive buffer plays no structural role in decision-making, forcing a monolithic head to recover timing, content, and redundancy judgements from compressed history alone. We propose a fundamentally different paradigm: \emph{proactive memory}, in which each memory tier does not merely store context but actively governs a distinct axis of the streaming decision. We instantiate this paradigm as \textbf{ProV-Mem}, a training-free framework comprising three constant-size, non-parametric proactive memory tiers sharing a frozen VLM backbone. A temporal-trigger memory actively controls \emph{when} to speak via deterministic gating over a compact timestamp state, with a silence-recovery for re-engagement; an instructional memory actively shapes \emph{what} to say by imposing a canonical surface-form prior that maximises visual grounding and aligns generations with downstream lexical filtering; and an episodic memory actively decides \emph{whether} to emit by caching the latest response and filtering redundancy via character-level similarity. Because each tier maintains constant-size state rather than raw-frame buffers, ProV-Mem eliminates coherence degradation and reduces cross-scenario adaptation to tuning a few interpretable scalars. We further show that these three proactive memory tiers can be optimized independently, with an approximation gap characterized by three stagewise discrepancy terms. Finally, we introduce Cost-Aware Area Under the Curve (CAUC), a metric that penalises redundancy, excessive frequency, and token inefficiency to close the evaluation loophole.

\textbf{Experimental success and contributions.}
Experiments on the ProactiveVideoQA benchmark \citep{wang2025proactivevideoqa} show that ProV-Mem consistently outperforms existing baselines. Without scenario-specific fine-tuning, the framework achieves state-of-the-art proactive performance while maintaining low response latency.
In summary, our main contributions are threefold:
\begin{itemize}[leftmargin=1.2em, itemsep=2pt]
 \item We propose ProV-Mem, a proactive memory framework factorising the proactive decision function into three independent memory stages to enable seamless cross-scenario adaptation.
  \item We expose a structural reward-hacking loophole in existing evaluation metrics and propose CAUC, a rigorous cost-aware metric that penalises redundancy, chatter frequency, and token inefficiency at each emission.
  \item We provide comprehensive empirical evidence showing that ProV-Mem achieves excellent performance on ProactiveVideoQA under both standard and cost-aware evaluation.
\end{itemize}

\section{Related Work}
\label{sec:related}
\textbf{Proactive-VQA and Streaming Video-Language Models}.
\label{sec:related-proactive-streaming}
Passive VQA s, in which the model consumes a full video before answering, has been extensively studied \citep{lei2020tvqa+, zhang2023video}. \citet{wang2025proactivevideoqa} depart from this paradigm by requiring frame-by-frame streaming with autonomous decisions on timing, content, and emission necessity. Broader benchmarks such as StreamingBench \citep{lin2026streamingbench} and OVOBench \citep{niu2025ovo} advance online video understanding but remain question-driven, leaving autonomous firing outside their protocol. Existing methods split into trained trigger-head systems such as MMDuet \citep{wang2025videollm} and VideoLLM-Online \citep{chen2024videollm}, and training-free adaptations via chunk-based \citep{zhang2024long, chen2025longvila} or prompt-based \citep{han2024free, wang2024videoagent} reuse of offline VLMs. A parallel line on long-context streaming, including StreamingVLM \citep{xu2026streamingvlm} and StreamingLLM \citep{xiao2024efficient}, scales to unbounded inputs but assumes passive interaction. The trained-head family degrades outside its training distribution, while training-free and long-context families lack autonomous firing, leaving a structurally complete, proactive controller as the gap our work targets.

\textbf{Prompt-Based Control and Memory-Augmented Models}
\label{sec:related-prompt-memory}
Since CLIP \citep{radford2021learning} and Flamingo \citep{alayrac2022flamingo}, a substantial line of work shapes frozen VLM behaviour at inference time via system prompts and in-context examples, where grounded-detail and persona prompts \citep{you2024ferret, maaz2024video} elicit specificity or stylistic surface form.
In parallel, a memory-augmented lineage attaches explicit memory modules read at inference to modulate generation, including MemoryBank \citep{zhong2024memorybank} which stores summarised interactions in a vector database, and Memorizing Transformers \citep{wu2022memorizing} and RETRO \citep{borgeaud2022improving} which retrieve nearest-neighbour examples during generation. These two lineages are typically pursued in isolation, and the memory-augmented one is moreover parametric, requiring target-domain training and retrieval-time operations that block the streaming forward pass; the absence of a single non-parametric formulation that unifies inference-time prompting and explicit memory while staying backbone-agnostic and streaming-compatible is precisely the gap our work targets.

\begin{figure}[!t]
  \centering
  \includegraphics[width=0.95\linewidth]{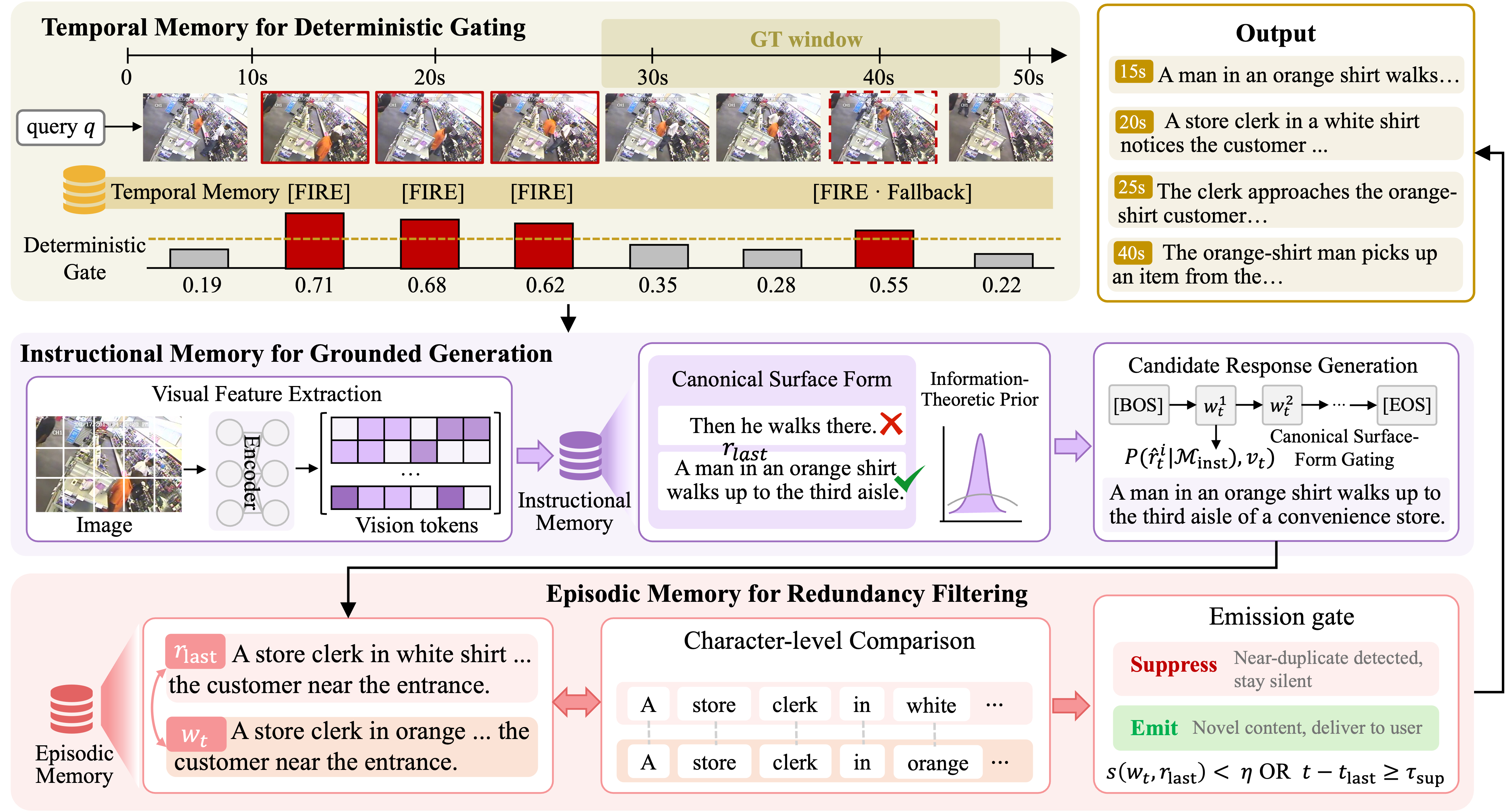}
  \caption{Framework of our ProV-Mem method.}
\vspace{-1.5em}
\label{fig:framework}
\end{figure}

\section{Method}
\label{sec:method}

This section presents the mathematical and structural design of ProV-Mem shown in Figure~\ref{fig:framework}. We first formalize Proactive-VQA under streaming causality constraints. We then show that the proactive decision can be decomposed into three coordinated stages, with an explicit suboptimality bound governed by measurable stagewise violation terms. Next, we describe the three lightweight memory tiers that implement this decomposition in practice. Finally, we introduce CAUC, a cost-aware refinement of PAUC that penalizes redundant, overly frequent, and unnecessarily long emissions.

\subsection{Problem Formulation}
\label{sec:formulation}

A Proactive-VQA instance provides an unbounded video stream
$\mathcal{V}=(x_1,x_2,\ldots,x_T)$ arriving frame by frame, where $x_t$ denotes
the visual frame received at step $t\in\{1,\ldots,T\}$. Although the underlying
stream is conceptually unbounded, $T$ denotes the finite prefix observed in an
evaluation run; thus, $\mathcal{V}=X_{\leq T}$, where
$X_{\leq t}=(x_1,\ldots,x_t)$ is the causal prefix available at step $t$. The
query $q$ is issued at $t=0$, while the ground-truth reference turns
$\{(s_j,e_j,r_j^*)\}_{j=1}^{n}$ are reserved for post-hoc evaluation only.
Here, $s_j$ and $e_j$ denote the start and end timestamps of the $j$-th target
window, and $r_j^*$ is its reference response.
At each step, the agent observes $X_{\leq t}$ and its emitted-response history
$H_{<t}=\{(t',a_{t'}):t'<t,\ a_{t'}\neq\varnothing\}$, and produces an
irrevocable action $a_t\in\Sigma^*\cup\{\varnothing\}$. Here, $\Sigma^*$ is the
set of textual responses over vocabulary $\Sigma$, and $\varnothing$ denotes an
internal no-emission action. The full causal state is
$S_{\leq t}=(X_{\leq t},q,H_{<t})$. The policy depends only on $S_{\leq t}$: future frames are inaccessible and emitted responses are irrevocable. For the
subsequent analysis, $P_\delta(\cdot\mid S_{\leq t})$ denotes the conditional distribution induced by this policy, ``No Reply'' is mapped to $\varnothing$ before evaluation.

\subsection{Decision Decomposition and Suboptimality Bound}
\label{sec:decomposition}

Predicting $a_t$ directly from the full causal state $S_{\leq t}$ entangles
three decisions: when to speak, what to say, and whether to emit. These
decisions depend on temporal firing state, immediate visual grounding, and
recent emitted text, respectively, motivating a modular stagewise process.

\textbf{Structural Decomposition.}
We equip the agent with three memory tiers around a shared VLM backbone.
Temporal-Trigger Memory $\mathcal{M}_{\mathrm{trig}}$ maintains
$\mathcal{M}_{\mathrm{trig},t}=(p_t,c_{\mathrm{sil},t})$, where
$p_t\in\{0,\ldots,K-1\}$ is the cadence phase modulo the trigger stride $K$,
and $c_{\mathrm{sil},t}$ counts consecutive regular trigger opportunities that
end without an emitted response.
It produces a binary trigger decision $f_t\in\{0,1\}$. Instructional Memory
$\mathcal{M}_{\mathrm{inst}}$ stores a fixed prompting prior that generates a
candidate response from the active window $V_t$, with encoded representation
$z_t=\phi(V_t)$. When $f_t=1$, the backbone generates
$\hat r_t\in\Sigma^*$; otherwise, $\hat r_t=\varnothing$. Episodic Memory
$\mathcal{M}_{\mathrm{ep}}$ stores
$\mathcal{M}_{\mathrm{ep},t}=(\bar r_{\mathrm{last},t},
\tau_{\mathrm{last},t}^{\mathrm{emit}})$, where
$\bar r_{\mathrm{last},t}$ is a fixed-length prefix of the most recent emitted
response and $\tau_{\mathrm{last},t}^{\mathrm{emit}}$ is its timestamp. It
produces an emission decision $e_t\in\{0,1\}$ when $f_t=1$, and $e_t=0$
otherwise. Although the deployed gates are deterministic,
$P_{\delta_{\mathrm{dec}}}$ denotes the distribution induced by the
three-stage decomposed process, with $P_{\mathrm{trig}}$,
$P_{\mathrm{gen}}$, and $P_{\mathrm{emit}}$ as its stagewise factors. The
latent decision chain factorizes as:
\begin{equation}
\small
\label{eq:decomp}
P_{\delta_{\mathrm{dec}}}(f_t,\hat r_t,e_t\mid S_{\leq t})
=
P_{\mathrm{trig}}(f_t\mid\mathcal{M}_{\mathrm{trig},t})
\cdot
P_{\mathrm{gen}}(\hat r_t\mid f_t,\mathcal{M}_{\mathrm{inst}},z_t)
\cdot
P_{\mathrm{emit}}(e_t\mid f_t,\hat r_t,\mathcal{M}_{\mathrm{ep},t}).
\end{equation}
The final action is $a_t=\hat r_t$ when $f_t=e_t=1$, and
$a_t=\varnothing$ otherwise. Thus, Eq.~\ref{eq:decomp} factorizes the latent
decision chain rather than the final action distribution.

\textbf{Conditional-Sufficiency Assumption.}
Eq.~\ref{eq:decomp} restricts each decision stage to its designated local
representation rather than the full causal state. We use this as an idealized
reference condition: the local representations preserve the conditional
distributions of the corresponding stage decisions. In practice, deviations
from this condition are captured by the stagewise discrepancies below.

\begin{assumption}[Conditional Sufficiency]
\label{asm:suff}
Replacing the full causal state with
$\mathcal{M}_{\mathrm{trig},t}$,
$(\mathcal{M}_{\mathrm{inst}},z_t)$, and
$(\hat r_t,\mathcal{M}_{\mathrm{ep},t})$ preserves the conditional
distributions of the trigger decision $f_t$, the candidate response
$\hat r_t$ given $f_t=1$, and the emission decision $e_t$ given
$f_t=1$ and $\hat r_t$, respectively.
\end{assumption}

\textbf{Suboptimality Bound.}
Let $\tau_t$ denote the wall-clock timestamp at step $t$, and let $g(t)$ denote
the index of the ground-truth target window containing $\tau_t$. If no target
window is active, we set $g(t)=\varnothing$ and
$r_{\varnothing}^*=\varnothing$. Let
$U(\delta)=\mathbb{E}[\sum_{t=1}^{T}u(a_t,r_{g(t)}^*)]$ denote the expected
cumulative utility of decision process $\delta$, where the expectation is taken
under the common-state counterfactual evaluation and
$u(\cdot,\cdot)\in[0,c]$ is a bounded per-step utility. We compare the joint and
decomposed processes on the same causal state at each step. Let
$\epsilon_{\mathrm{trig}}$, $\epsilon_{\mathrm{gen}}$, and
$\epsilon_{\mathrm{emit}}$ denote the average probabilities that the two
processes disagree at the trigger, generation, and emission stages,
respectively. Let $\delta_{\mathrm{joint}}^*$ denote the optimal process with
access to the full causal state, and let $\delta_{\mathrm{dec}}^*$ denote the
process induced by the three local memory stages.

\begin{lemma}[Stagewise Approximation Bound]
\label{lem:bound}
Under the common-state counterfactual evaluation,
$U(\delta_{\mathrm{joint}}^*)-
U(\delta_{\mathrm{dec}}^*)\leq
2cT(\epsilon_{\mathrm{trig}}+
\epsilon_{\mathrm{gen}}+
\epsilon_{\mathrm{emit}})$.
\end{lemma}

The bound follows because a final-action disagreement requires a disagreement
at at least one stage, while the per-step utility is bounded by $c$. The
surface-semantic collision rate $\rho$ denotes the probability that a
semantically novel candidate is suppressed by the episodic gate because of high
lexical similarity to the most recent emitted response. It is reported
separately as an interpretable diagnostic and is not treated as the complete
emission-stage discrepancy. The full derivation is provided in
Appendix~\ref{app:proof}.

\subsection{ProV-Mem: Co-Designed Lightweight Memory Tiers}
\label{sec:tiers}

We now instantiate Eq.~\ref{eq:decomp} using deterministic memory updates
around a frozen VLM backbone. The input stream is processed frame by frame, and $K$ denotes the regular trigger stride measured in sampled frames, such that a regular trigger opportunity occurs every $K$ sampled frames. Throughout this
subsection, $t$ indexes streaming steps, while regular trigger opportunities
occur whenever the cadence phase satisfies $p_t=0$.

\textbf{Temporal Memory for Deterministic Gating.}
The temporal-trigger memory maintains
$\mathcal{M}_{\mathrm{trig},t}=(p_t,c_{\mathrm{sil},t})$, where $p_t$ is the
cadence phase modulo the trigger stride $K$ and $c_{\mathrm{sil},t}$ counts
consecutive regular trigger opportunities without an emitted response. The
deployed trigger gate is deterministic:
$f_t
=
\mathds{1}\!\left[
p_t=0
\;\lor\;
c_{\mathrm{sil},t}\geq\kappa
\right]$.
Thus, the backbone is queried at regular intervals of $K$ sampled frames. When
the silence counter reaches $\kappa$, the recovery condition overrides the
regular cadence and initiates an additional query to prevent prolonged
inactivity. Here, $\kappa$ controls the number of consecutive silent regular
trigger opportunities tolerated before recovery.

\textbf{Instructional Memory for Grounded Generation.}
The instructional memory $\mathcal{M}_{\mathrm{inst}}$ stores the fixed
prompting prior used after a trigger. Given the active visual window $V_t$ and
its encoded representation $z_t=\phi(V_t)$, the frozen backbone generates
$\hat r_t\in\Sigma^*$. The prompt encourages concise, self-contained responses
grounded in the current window and promotes a consistent surface form for
semantically similar events, improving compatibility with the downstream
lexical filter.

\textbf{Episodic Memory for Redundancy Filtering.}
The episodic memory stores
$\mathcal{M}_{\mathrm{ep},t}=
(\bar r_{\mathrm{last},t},\tau_{\mathrm{last},t}^{\mathrm{emit}})$, where
$\bar r_{\mathrm{last},t}$ is the first $L$ characters of the most recent
emitted response and $\tau_{\mathrm{last},t}^{\mathrm{emit}}$ is its
wall-clock timestamp. At the beginning of a stream, we initialize
$\bar r_{\mathrm{last},0}=\varnothing$ and
$\tau_{\mathrm{last},0}^{\mathrm{emit}}=-\infty$. For a candidate response
$\hat r_t$, let $\bar{\hat r}_t=\hat r_t[:L]$ denote its comparison prefix. We
compute the normalized prefix-level lexical similarity as
\begin{equation}
\label{eq:sim}
s(\bar{\hat r}_t,\bar r_{\mathrm{last},t})
=
{
2\,\mathrm{LCS}\left(\bar{\hat r}_t,\bar r_{\mathrm{last},t}\right)
}/\left({
|\bar{\hat r}_t|+|\bar r_{\mathrm{last},t}|
}\right),
\end{equation}
where $\mathrm{LCS}(u,v)$ denotes the length of the longest common subsequence
between strings $u$ and $v$. If no previous response exists, we set the
similarity to zero. The deterministic emission gate is
\begin{equation}
\label{eq:emit}
e_t
=
\mathds{1}\!\left[
f_t=1
\;\wedge\;
\left(
s(\bar{\hat r}_t,\bar r_{\mathrm{last},t})<\eta
\;\lor\;
\tau_t-\tau_{\mathrm{last},t}^{\mathrm{emit}}
>
\tau_{\mathrm{sup}}
\right)
\right],
\end{equation}
where $\eta\in(0,1)$ is the similarity threshold and
$\tau_{\mathrm{sup}}$ is the suppression cooldown. Thus, lexically similar
responses can be emitted again only after a sufficiently long temporal gap.

\textbf{Inter-Tier Feedback Loop.}
At each streaming step, the cadence phase advances as
$p_{t+1}=(p_t+1)\bmod K$. When $f_t=0$, the action is
$a_t=\varnothing$ and the episodic state remains unchanged. When $f_t=1$ and
$e_t=1$, the system emits $a_t=\hat r_t$, updates the episodic memory as
$(\bar r_{\mathrm{last},t+1},
\tau_{\mathrm{last},t+1}^{\mathrm{emit}})
\leftarrow
(\bar{\hat r}_t,\tau_t)$, and resets the silence counter to zero. When a
regular trigger opportunity produces no emission, the silence counter is
incremented. Once $c_{\mathrm{sil},t}\geq\kappa$, the recovery query is issued
with $P_{\mathrm{sys}}^{\mathrm{force}}$ and the silence counter is reset after
that recovery attempt. Since the temporal state, prompting prior, and episodic
prefix remain fixed-size or bounded-size, ProV-Mem operates as a lightweight
training-free wrapper around the frozen VLM backbone.

\subsection{CAUC: Cost-Aware Evaluation Metric}
\label{sec:cauc}

PAUC evaluates the timeliness and semantic quality of proactive responses, but
does not explicitly account for redundancy, emission frequency, or response
length. We therefore introduce \textbf{C}ost-\textbf{A}ware \textbf{A}rea
\textbf{U}nder the \textbf{C}urve (CAUC), which reweights the quality score used
by PAUC with three explicit penalties. For an emitted response $a_t$, let $q_t$
denote the semantic quality assigned by the official PAUC evaluator to the
cumulative response history available after this emission. The redundancy penalty
is $P_{\mathrm{red},t}=0$ when $H_{<t}$ is empty and otherwise
$P_{\mathrm{red},t}=
\max_{(t',a)\in H_{<t}}s(a_t[:L],a[:L])$, where $s(\cdot,\cdot)$ is the
normalized prefix-level LCS similarity defined in Eq.~\ref{eq:sim}. Unlike the
online episodic gate, which stores only the most recent response, CAUC uses the
complete emitted history $H_{<t}$ for evaluation. The frequency penalty is
$P_{\mathrm{freq},t}=0$ for the first emission and
$P_{\mathrm{freq},t}=\exp(-\Delta\tau_t/\tau_{\mathrm{cooldown}})$ otherwise,
where $\Delta\tau_t$ is the elapsed wall-clock time since the most recent
emitted response. The length penalty is
$P_{\mathrm{len},t}=
\max(0,|a_t|_{\mathrm{tok}}-L_{\mathrm{ref}})/L_{\mathrm{ref}}$, where
$|a_t|_{\mathrm{tok}}$ denotes the output length in tokens. We define the
cost-adjusted quality as
\begin{equation}
\label{eq:cauc}
\tilde q_t
=
\left[
q_t
\left(
1-\alpha P_{\mathrm{red},t}
-\beta P_{\mathrm{freq},t}
-\gamma P_{\mathrm{len},t}
\right)
\right]_+,
\end{equation}
where $\alpha,\beta,\gamma\in[0,1]$ are penalty weights. CAUC reuses the temporal integration procedure and timeliness--correctness trade-off parameter $\omega$ of PAUC, replacing the original quality score $q_t$ with $\tilde q_t$. Thus, CAUC preserves the temporal semantics of PAUC while explicitly discounting redundant, overly frequent, and unnecessarily long responses. For presentation, the resulting CAUC score is scaled to the range $[0,100]$.

\section{Experiments}
\label{sec:exp}

\subsection{Setup}

\textbf{Experimental Setup and Evaluation.}
We evaluate ProV-Mem on the four scenarios of ProactiveVideoQA
\citep{wang2025proactivevideoqa}, reporting PAUC at $\omega = 0.5$ alongside our proposed CAUC, using the official GPT-4.1 evaluator. We additionally report the average per-response inference time (Time), measured as the wall-clock latency from processing the current trigger window to producing the final emitted response. Under a frozen Qwen2.5-VL-7B backbone \citep{qwen2025qwen25vl}, trigger parameters $(K,\kappa)$ and episodic memory settings are adapted per scenario: a suppression cooldown $\tau_{\mathrm{sup}} = 30$\,s on EGO for recurring substeps, a silence-recovery $\kappa = 4$ on TV, and an event cooldown $\tau_{\mathrm{event}} = 8$\,s with scouting prompt $P_{\mathrm{scout}}$ on VAD. Full configurations are provided in Appendix~\ref{app-exp:setting}.

\textbf{Baselines.}
We compare against three groups of baselines.  \textbf{1) Offline closed-source VLMs} include the proprietary models originally evaluated in ProactiveVideoQA \citep{wang2025proactivevideoqa}, namely GPT-4o-mini \citep{achiam2023gpt}, Gemini-2.5-Pro, Gemini-2.5-Flash \citep{team2024gemini}, and GLM-4.5V \citep{glm2024chatglm}. 
\textbf{2) Offline open-source VLMs} are adapted to the proactive setting via chunked inference, including Qwen2.5-VL-7B \citep{qwen2025qwen25vl}, LLaVA-OV-7B \citep{li2024llava}, LongVA-7B \citep{zhang2024long}, InternVL-2.5-8B \citep{chen2024internvl25}, and IXC-2.5-7B \citep{zhang2024internlm}.
\textbf{3) Online Streaming VLMs} with trained triggering mechanisms include MMDuet and its variant denoted in the original paper as MMDuet + rm. ass., where rm. ass. refers to removing previous assistant responses.\citep{wang2025videollm}, MMDuet2 (RL) \citep{wang2026mmduet2}, and VideoLLM-Online \citep{chen2024videollm}. 
To ensure strict evaluation alignment and fair comparison, we independently reproduce and rerun all baseline systems under a unified evaluation protocol.

\begin{table*}[t]
\centering
\caption{Main results on ProactiveVideoQA.
Best results are in \textbf{bold}.}
\vspace{-0.5em}
\label{tab:main}
\renewcommand{\arraystretch}{1} 
\resizebox{0.95\textwidth}{!}{%
\begin{tabular}{lccccccccr}
\toprule
\multirow{2}{*}{System}
  & \multicolumn{4}{c}{PAUC ($\omega{=}0.5$, $\uparrow$)}
  & \multicolumn{4}{c}{CAUC ($\uparrow$)}
  & \multirow{2}{*}{Time (s) $\downarrow$} \\
\cmidrule(lr){2-5}\cmidrule(lr){6-9}
  & WEB & EGO & TV & VAD & WEB & EGO & TV & VAD & \\
\midrule
\multicolumn{10}{l}{\textit{\textbf{Offline Closed-Source VLMs}}} \\
\rowcolor{gray!8}
GLM-4.5V          & 40.60 & 50.28 & 31.99 & 36.57 & 29.29 & 38.16 & 28.22 & 30.96 & 45.28 \\
\rowcolor{gray!8}
Gemini-2.5-Pro    & 36.88 & 41.27 & 32.62 & 36.92 & 26.80 & 32.25 & 29.35 & 30.51 & 20.00 \\
\rowcolor{gray!8}
Gemini-2.5-Flash  & 40.15 & 42.11 & 28.81 & 35.98 & 29.31 & 33.23 & 26.46 & 30.16 & 17.20 \\
\rowcolor{gray!8}
GPT-4o-mini       & 38.03 & 33.48 & 27.83 & 35.28 & 28.46 & 29.72 & 26.37 & 30.75 & 21.52 \\
\midrule
\multicolumn{10}{l}{\textit{\textbf{Offline Open-Source VLMs}}} \\
\rowcolor{gray!4}
InternVL-2.5-8B   & 36.57 & 48.46 & 32.84 & 29.32 & 26.84 & 36.82 & 27.34 & 26.16 & 8.12 \\
\rowcolor{gray!4}
Qwen2.5-VL-7B     & 35.82 & 39.27 & 29.46 & 23.25 & 25.65 & 31.36 & 23.83 & 19.10 & 7.46 \\
\rowcolor{gray!4}
LLaVA-OV-7B       & 34.31 & 23.57 & 27.23 & 26.75 & 25.26 & 22.81 & 25.64 & 24.82 & 94.87 \\
\rowcolor{gray!4}
LongVA-7B         & 27.11 & 25.46 & 24.95 & 24.88 & 25.75 & 25.22 & 24.88 & 24.86 & 6.00 \\
\rowcolor{gray!4}
IXC-2.5-7B        & 30.14 & 13.19 & 14.68 & 16.36 & 22.48 & 10.51 & 12.04 & 13.90 & 20.23 \\
\midrule
\multicolumn{10}{l}{\textit{\textbf{Online Streaming VLMs}}} \\
VideoLLM-Online        & 24.35 & 25.01 & 22.57 & 25.00 & 22.91 & 24.98 & 22.21 & 25.00 & 10.69 \\
MMDuet                 & 34.78 & 36.56 & 20.58 & 22.76 & 29.00 & 26.84 & 14.53 & 18.64 & 3.03 \\
MMDuet + rm.\ ass.     & 39.85 & 46.37 & 26.33 & \textbf{36.87}& 26.38 & 25.87 & 18.42 & 18.99 & 1.90 \\
MMDuet2 (RL)           & 54.83 & 31.58 & 41.69 & 28.89 & 40.29 & 28.41 & 34.89 & 25.97 & 20.80 \\
\midrule
\rowcolor{blue!12}
\textbf{ProV-Mem (Ours)} & \textbf{57.90} & \textbf{70.87} & \textbf{57.43} & 32.10
                         & \textbf{40.59} & \textbf{38.20} & \textbf{40.22} & \textbf{31.68}
                         & \textbf{1.30} \\
\hdashline

\bottomrule
\end{tabular}}
\end{table*}

\subsection{Main Results}
To evaluate the efficacy of our ProV-Mem, we compare ProV-Mem against representative baselines. As shown in Table~\ref{tab:main}, the results expose the structural limitations of parametric streaming architectures across three dimensions. First, online models equipped with explicit streaming supervision are consistently outperformed by static offline baselines on multiple scenarios, a failure that traces to a fundamental architectural constraint: parametric designs must compress long-horizon visual history into a fixed-capacity hidden state to sustain low latency, and this compression irreversibly degrades the representational fidelity required for complex video reasoning. ProV-Mem circumvents this capacity-latency tension by externalising temporal state management into non-parametric memory tiers, leaving the representational capacity of the frozen backbone intact. Second, the cross-domain performance collapse exhibited by RL-trained trigger policies reveals a deeper generalisation deficiency: reinforcement learning overfits trigger thresholds to the temporal statistics of the training distribution, and any shift in event density or temporal structure between domains causes the learned trigger head to fail. Third, the systematic divergence between PAUC and CAUC across competing baselines provides evidence of metric reward-hacking: because PAUC does not penalise redundant emissions, several systems opportunistically over-trigger on detected segments to inflate their scores, while the cost-aware CAUC metric corrects this inflation and restores the true performance ordering. The fact that ProV-Mem simultaneously leads on both metrics across three of the four scenarios suggsets that its gains arise from precise content management and semantic deduplication rather than unconstrained firing frequency. 
The main exception is the VAD PAUC score, which suggests that anomaly detection may require stronger parametric event priors than surface-level lexical memory can provide. This result marks the current boundary of the non-parametric design and points to the temporal memory tier as the most promising direction for future extension.

\subsection{Ablation Study}
\begin{wraptable}[12]{R}{8.2cm}
\centering
\small
\begingroup
\setlength{\tabcolsep}{1.8pt}   
\renewcommand{\arraystretch}{1.12}
\caption{Ablation study on ProactiveVideoQA.}
\label{tab:ablation_wrap}
\begin{tabular}{l *{8}{c}}
\toprule
\multirow{2}{*}{\textbf{Variant}}
  & \multicolumn{4}{c}{\textbf{PAUC} ($\omega{=}0.5$, $\uparrow$)}
  & \multicolumn{4}{c}{\textbf{CAUC} ($\uparrow$)} \\
\cmidrule(lr){2-5}\cmidrule(lr){6-9}
  & WEB & EGO & TV & VAD & WEB & EGO & TV & VAD \\
\midrule
\rowcolor{blue!12}
\textbf{Full}
  & {57.90} &{70.87} & {57.43} & {32.10}
  & \textbf{40.59} & \textbf{38.20} & \textbf{40.22} & \textbf{31.68} \\
\midrule
\rowcolor{gray!10}
$-\mathcal{M}_{\mathrm{trig}}$
  & 37.51 & 66.25 & 55.91 & 32.76
  & 33.10 & 31.65 & 40.16 & 27.19 \\
\rowcolor{gray!5}
$-\mathcal{M}_{\mathrm{inst}}$
  & 45.06 & 58.93 & 55.03 & 31.08
  & 34.92 & 35.34 & 39.14 & 27.20 \\
\rowcolor{gray!10}
$-\mathcal{M}_{\mathrm{ep}}$
  & \textbf{60.14} & \textbf{71.78} & \textbf{62.64} & \textbf{33.38}
  & 40.12 & 37.80 & 38.83 & 26.33 \\
\bottomrule
\end{tabular}
\endgroup
\end{wraptable}

To verify the empirical gains of ProV Mem, we conduct the ablation study shown in Table~\ref{tab:ablation_wrap}. The Full model delivers the strongest overall cost-aware trade-off, achieving the highest CAUC in all four evaluation scenarios. Although removing the episodic memory tier, $\mathcal{M}_{\mathrm{ep}}$, can slightly improve raw PAUC by allowing more frequent emissions, this benefit comes with substantially higher redundancy and response cost. The results therefore suggest that episodic deduplication plays an important role in maintaining efficiency under continuous streaming evaluation, and that the Full configuration provides the most reliable balance between utility and cost.

The other two tiers govern complementary aspects of the decision process. Removing the instructional memory tier, $\mathcal{M}_{\mathrm{inst}}$, leads to a consistent decline in both metrics across scenarios, indicating that stable prompt-level guidance is important for maintaining response quality. Removing the temporal trigger tier, $\mathcal{M}_{\mathrm{trig}}$, reveals stronger temporal sensitivity in domains with dense or rapidly changing events. On WEB, the absence of this tier markedly reduces PAUC, showing that dense event streams require timely state updates. On EGO, removing $\mathcal{M}_{\mathrm{trig}}$ lowers both PAUC and CAUC while increasing repetition, which suggests that temporal triggering is also important for controlling redundant backbone invocations. Taken together, these findings indicate that the three memory tiers serve complementary functions, and that the Full configuration offers the most stable performance.

\begin{figure}[h]
  \centering
  \includegraphics[width=\linewidth]{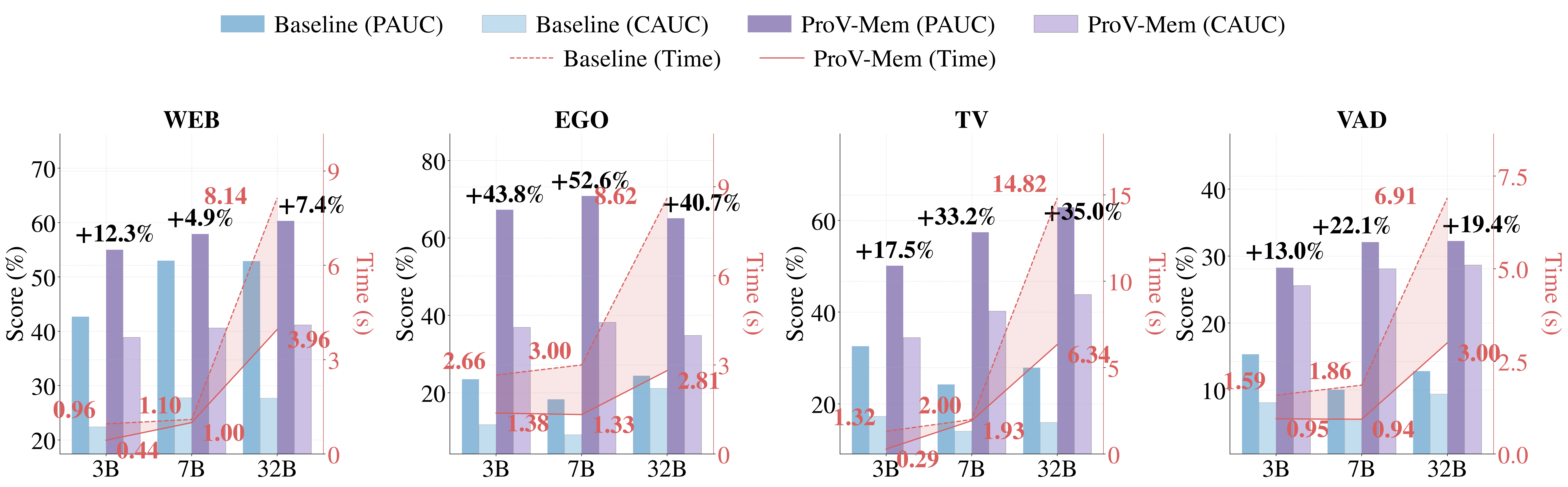}
  \caption{Backbone scaling on ProactiveVideoQA with the ProV-Mem.}
  \label{fig:backbone}
\vspace{-0.6em}
\end{figure}

\subsection{Effect of Backbone Scale}

To assess whether ProV-Mem generalizes across backbone scales, we evaluate Qwen2.5-VL at 3B, 7B, and 32B parameters with fixed memory hyperparameters. Figure~\ref{fig:backbone} reveals two key findings.
First, the baseline exhibits non-monotonic scaling curves on several scenarios, a symptom of its monolithic design. As backbone capacity grows, the implicit context buffer must simultaneously handle temporal reasoning and semantic generation, and these two objectives interfere under the pressure of more complex representations, causing accuracy to plateau or regress on high-density streams. ProV-Mem, by contrast, scales monotonically across all scenarios because the non-parametric memory tiers absorb all temporal state management externally, allowing the backbone to translate increased capacity purely into stronger semantic grounding without concurrently maintaining memory state. Second, the computational efficiency advantage conferred by the temporal-trigger gate compounds with scale. By suppressing redundant backbone invocations on silent or uninformative windows, ProV-Mem reduces its effective inference workload sub-linearly relative to the baseline, and this saving becomes increasingly pronounced as the per-forward-pass cost rises with parameter count. Taken together, these results demonstrate that ProV-Mem functions as a structural scaffold that converts backbone scaling into monotonic quality improvement while bounding the growth rate of computational overhead, two properties that parametric streaming designs cannot provide at the architectural level regardless of the underlying model size.

\subsection{Sensitivity of Memory Hyperparameters}
\label{sec:sensitivity}

\begin{figure}[!t]
  \centering
  \includegraphics[width=\linewidth]{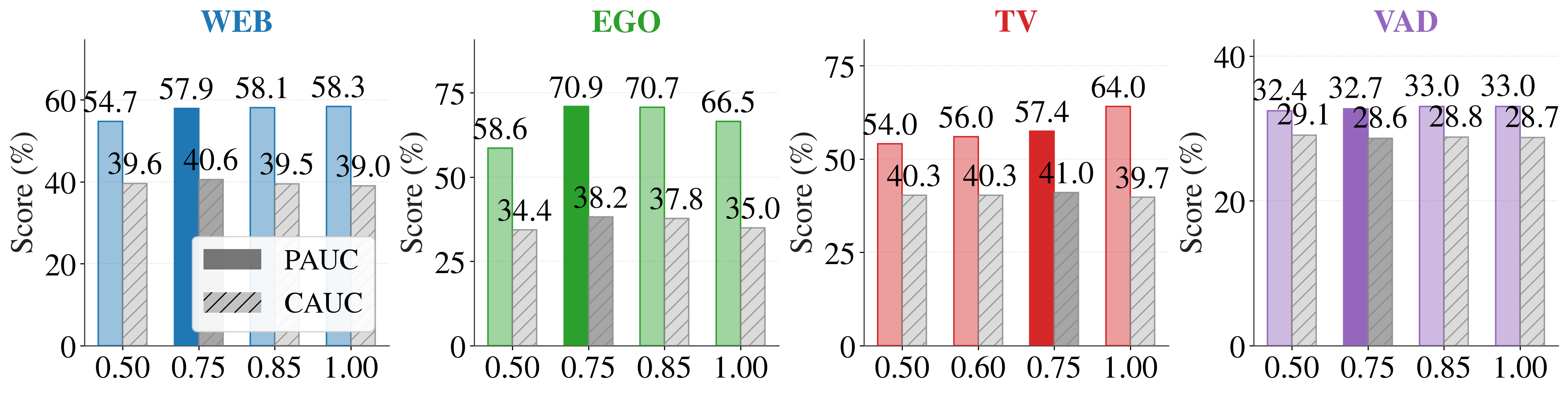}
  \caption{Effect of the similarity threshold $\eta$ across  ProactiveVideoQA scenarios.}
  \label{fig:h2}
\vspace{-0.5em}
\end{figure}

\begin{figure}[!t]
  \centering
  \includegraphics[width=\linewidth]{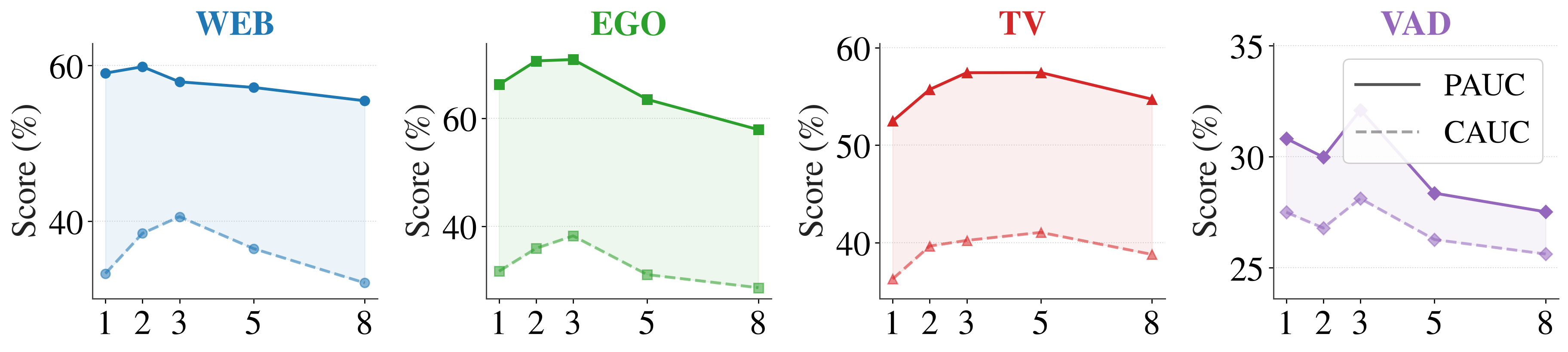}
  \caption{Effect of the regular trigger stride $K$ across all four ProactiveVideoQA scenarios.}
  \label{fig:K}
\vspace{-1.3em}
\end{figure}

\textbf{Effect of similarity threshold.}
The similarity threshold governs a fundamental trade-off between two error terms in the suboptimality bound. Tightening it reduces the collision rate $\rho$ but risks suppressing semantically novel candidates and inflating $\epsilon_{gen}$, while relaxing it admits near-duplicate emissions at the expense of cost efficiency. Figure~\ref{fig:h2} reveals that this trade-off manifests differently across domains in ways that carry theoretical weight beyond hyperparameter selection. The most diagnostic finding is on EGO, where raw PAUC also degrades when deduplication is fully disabled, demonstrating that procedural repetitions actively dilute the matching signal by occupying the candidate set's allocation budget rather than merely inflating redundancy cost. TV provides the clearest evidence of the PAUC-CAUC asymmetry: the subtitle-anchored event structure allows repetitive fires to achieve coverage through temporal saturation rather than semantic precision, a form of metric exploitation that CAUC is specifically designed to penalise. 
VAD's complete insensitivity to $\eta$ suggests that its sparse event structure is governed by the temporal cooldown $\tau_{\text{event}}$ rather than the per-window similarity threshold, making $\eta$ a secondary factor in this domain.
These results validate $\eta=0.75$ as a robust cross-scenario default and establish that PAUC and CAUC measure fundamentally different quality axes that cannot be collapsed without discarding diagnostic information.

\textbf{Effect of trigger stride.}
The trigger stride $K$ determines how often the model refreshes its visual evidence, creating a trade-off between temporal coverage and query efficiency. As shown in Figure~\ref{fig:K}, WEB and EGO are most sensitive to larger strides: their short, dense events require frequent trigger opportunities, and less frequent updates can leave relevant visual changes unobserved until after the event has passed. TV exhibits a broader plateau because subtitle-aligned
events typically persist over longer intervals, making its response quality less sensitive to modest delays in visual updates. VAD is comparatively stable over the tested range, since the scouting stage provides an additional event-level filter before full response generation. These differences show that trigger
stride should match the temporal persistence of the target events. Across the four scenarios, $K=3$ provides the most favorable balance between event coverage and redundant backbone queries.

\begin{figure}[!t]
  \centering
  \includegraphics[width=\linewidth]{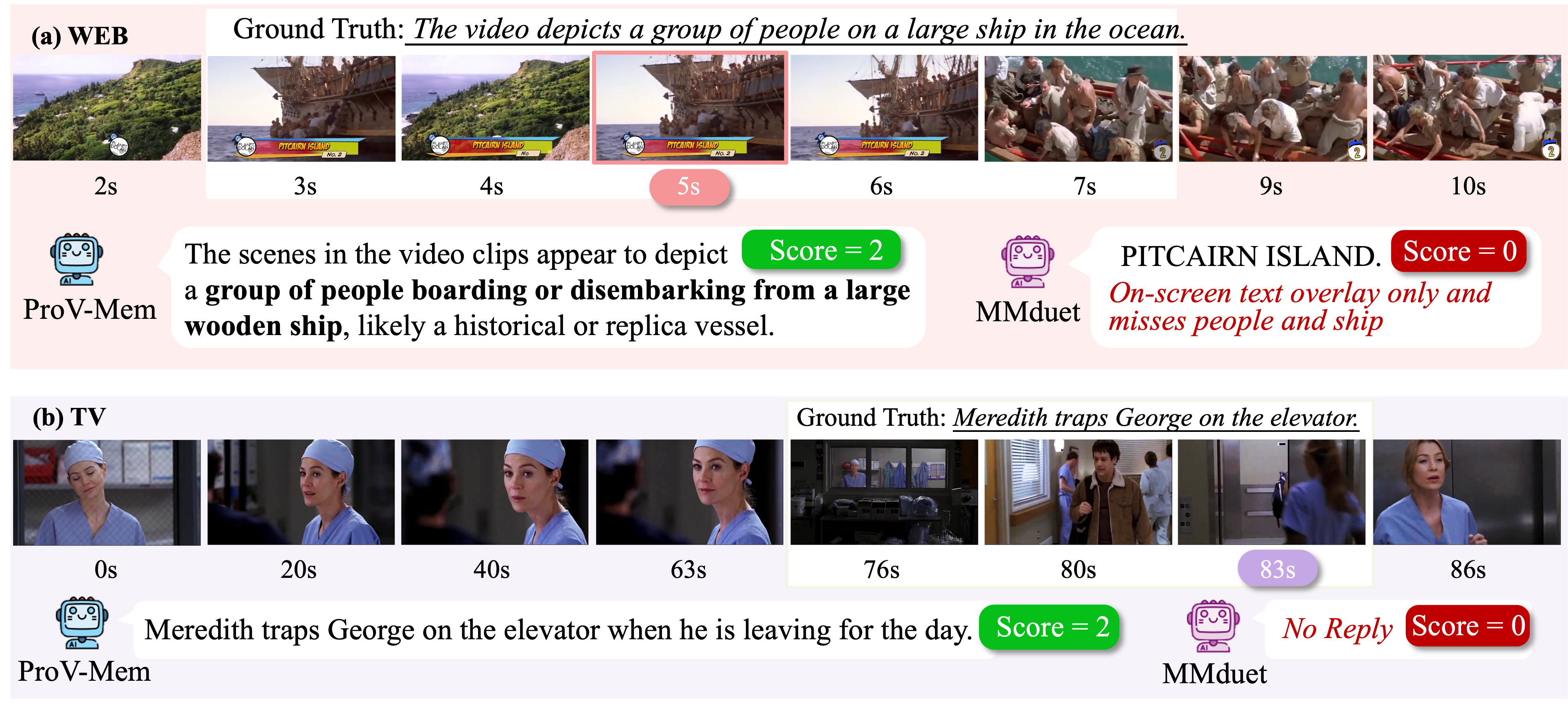}
  \caption{Comparison of (a) MMDuet and (b) our ProV-Mem method.}
\vspace{-1.6em}
\label{fig:case}
\end{figure}

\subsection{Case Study}
Figure~\ref{fig:case} illustrates how the memory tiers support different
proactive decisions. In case (a), the instructional memory guides ProV-Mem to
produce a canonical, visually grounded description of people on a large ship at
5\,s, whereas MMDuet follows the on-screen text overlay and produces a response
that is less informative with respect to the reference event. This example
illustrates the role of the instructional prior in shaping what to say from the
active visual window. In case (b), the temporal-trigger memory detects a
prolonged sequence of non-emitting trigger opportunities and activates the
recovery mechanism, enabling ProV-Mem to produce an accepted response for the
interpersonal event at 83\,s. The episodic memory evaluates each candidate
against the most recent emitted response before delivery, preventing the
recovery mechanism from simply increasing redundant emissions. Together, these
examples illustrate the complementary roles of the three tiers in controlling
when to query, what to generate, and whether to emit.

\section{Conclusion}
We presented ProV-Mem, a framework for proactive video reasoning that requires no additional training and organizes decision-making through a three-stage nonparametric memory hierarchy over a frozen VLM backbone. Our contributions are threefold. First, we formulate proactive reasoning as a decomposition of triggering, candidate generation, and emission, which gives the problem a clear and modular structure. Second, we provide a formal suboptimality bound that justifies this decomposition as a principled approximation to joint decision-making over the full causal state. Third, we introduce CAUC, a cost-aware evaluation metric that penalizes redundant, overly frequent, and verbose responses. Experiments on ProactiveVideoQA show that ProV-Mem achieves excellent performance with low latency, consistently improves the utility-cost trade-off, and establishes a stronger Pareto frontier without scenario-specific fine-tuning. These results show that careful memory organization and stagewise decision design are sufficient to produce effective proactive behavior even when the backbone model remains frozen.

\section*{AI use statement}
In this work, generative AI tools (large language models) were used solely for English grammar polishing and LaTeX formatting of the manuscript. Additionally, GPT-4.1 was employed as the reference semantic judge in our evaluation pipeline, strictly adhering to Section~4.1 and the official ProactiveVideoQA evaluation protocol \citep{wang2025proactivevideoqa}.

\bibliography{iclr2027_conference}

@article{wang2025proactivevideoqa,
  title={Proactivevideoqa: A comprehensive benchmark evaluating proactive interactions in video large language models},
  author={Wang, Yueqian and Meng, Xiaojun and Wang, Yifan and Zhang, Huishuai and Zhao, Dongyan},
  journal={arXiv preprint arXiv:2507.09313},
  year={2025}
}

@inproceedings{wang2025videollm,
  title={VideoLLM Knows When to Speak: Enhancing Time-Sensitive Video Comprehension with Video-Text Duet Interaction Format.},
  author={Wang, Yueqian and Meng, Xiaojun and Wang, Yuxuan and Liang, Jianxin and Wei, Jiansheng and Zhang, Huishuai and Zhao, Dongyan},
  booktitle={EMNLP (Findings)},
  pages={6338--6359},
  year={2025}
}

@inproceedings{lin2024video,
  title={Video-llava: Learning united visual representation by alignment before projection},
  author={Lin, Bin and Ye, Yang and Zhu, Bin and Cui, Jiaxi and Ning, Munan and Jin, Peng and Yuan, Li},
  booktitle={Proceedings of the 2024 conference on empirical methods in natural language processing},
  pages={5971--5984},
  year={2024}
}

@inproceedings{maaz2024video,
  title={Video-chatgpt: Towards detailed video understanding via large vision and language models},
  author={Maaz, Muhammad and Rasheed, Hanoona and Khan, Salman and Khan, Fahad},
  booktitle={Proceedings of the 62nd Annual Meeting of the Association for Computational Linguistics (Volume 1: Long Papers)},
  pages={12585--12602},
  year={2024}
}

@inproceedings{xu2016msr,
  title={Msr-vtt: A large video description dataset for bridging video and language},
  author={Xu, Jun and Mei, Tao and Yao, Ting and Rui, Yong},
  booktitle={Proceedings of the IEEE conference on computer vision and pattern recognition},
  pages={5288--5296},
  year={2016}
}

@inproceedings{lei2018tvqa,
  title={Tvqa: Localized, compositional video question answering},
  author={Lei, Jie and Yu, Licheng and Bansal, Mohit and Berg, Tamara},
  booktitle={Proceedings of the 2018 conference on empirical methods in natural language processing},
  pages={1369--1379},
  year={2018}
}

@article{driess2023palm,
  title={Palm-e: An embodied multimodal language model},
  author={Driess, Danny and Xia, Fei and Sajjadi, Mehdi SM and Lynch, Corey and Chowdhery, Aakanksha and Ichter, Brian and Wahid, Ayzaan and Tompson, Jonathan and Vuong, Quan and Yu, Tianhe and others},
  journal={arXiv preprint arXiv:2303.03378},
  year={2023}
}

@inproceedings{o2024open,
  title={Open x-embodiment: Robotic learning datasets and rt-x models: Open x-embodiment collaboration 0},
  author={O’Neill, Abby and Rehman, Abdul and Maddukuri, Abhiram and Gupta, Abhishek and Padalkar, Abhishek and Lee, Abraham and Pooley, Acorn and Gupta, Agrim and Mandlekar, Ajay and Jain, Ajinkya and others},
  booktitle={2024 IEEE International Conference on Robotics and Automation (ICRA)},
  pages={6892--6903},
  year={2024},
  organization={IEEE}
}

@inproceedings{xu2026streamingvlm,
  title={Streamingvlm: Real-time understanding for infinite video streams},
  author={Xu, Ruyi and Xiao, Guangxuan and Chen, Yukang and He, Liuning and Peng, Kelly and Lu, Yao and Han, Song},
  booktitle={International Conference on Learning Representations},
  volume={2026},
  pages={61463--61475},
  year={2026}
}

@inproceedings{chen2024videollm,
  title={Videollm-online: Online video large language model for streaming video},
  author={Chen, Joya and Lv, Zhaoyang and Wu, Shiwei and Lin, Kevin Qinghong and Song, Chenan and Gao, Difei and Liu, Jia-Wei and Gao, Ziteng and Mao, Dongxing and Shou, Mike Zheng},
  booktitle={2024 IEEE/CVF Conference on Computer Vision and Pattern Recognition (CVPR)},
  pages={18407--18418},
  year={2024},
  organization={IEEE}
}

@article{zhang2024long,
  title={Long context transfer from language to vision},
  author={Zhang, Peiyuan and Zhang, Kaichen and Li, Bo and Zeng, Guangtao and Yang, Jingkang and Zhang, Yuanhan and Wang, Ziyue and Tan, Haoran and Li, Chunyuan and Liu, Ziwei},
  journal={arXiv preprint arXiv:2406.16852},
  year={2024}
}

@article{zhang2024internlm,
  title={Internlm-xcomposer-2.5: A versatile large vision language model supporting long-contextual input and output},
  author={Zhang, Pan and Dong, Xiaoyi and Zang, Yuhang and Cao, Yuhang and Qian, Rui and Chen, Lin and Guo, Qipeng and Duan, Haodong and Wang, Bin and Ouyang, Linke and others},
  journal={arXiv preprint arXiv:2407.03320},
  year={2024}
}

@inproceedings{chen2025longvila,
  title={Longvila: Scaling long-context visual language models for long videos},
  author={Chen, Yukang and Xue, Fuzhao and Li, Dacheng and Hu, Qinghao and Zhu, Ligeng and Li, Xiuyu and Fang, Yunhao and Tang, Haotian and Yang, Shang and Liu, Zhijian and others},
  booktitle={International Conference on Learning Representations},
  volume={2025},
  pages={18227--18246},
  year={2025}
}

@article{li2024llava,
  title={Llava-onevision: Easy visual task transfer},
  author={Li, Bo and Zhang, Yuanhan and Guo, Dong and Zhang, Renrui and Li, Feng and Zhang, Hao and Zhang, Kaichen and Zhang, Peiyuan and Li, Yanwei and Liu, Ziwei and others},
  journal={arXiv preprint arXiv:2408.03326},
  year={2024}
}

@inproceedings{lin2026streamingbench,
  title={Streamingbench: Assessing the gap for mllms to achieve streaming video understanding},
  author={Lin, Junming and Fang, Zheng and Chen, Chi and Cheng, Haoxuan and Wan, Zihao and Luo, Fuwen and Wang, Ziyue and Li, Peng and Liu, Yang and Sun, Maosong},
  booktitle={ICASSP 2026-2026 IEEE International Conference on Acoustics, Speech and Signal Processing (ICASSP)},
  pages={12147--12151},
  year={2026},
  organization={IEEE}
}

@inproceedings{niu2025ovo,
  title={Ovo-bench: How far is your video-llms from real-world online video understanding?},
  author={Niu, Junbo and Li, Yifei and Miao, Ziyang and Ge, Chunjiang and Zhou, Yuanhang and He, Qihao and Dong, Xiaoyi and Duan, Haodong and Ding, Shuangrui and Qian, Rui and others},
  booktitle={2025 IEEE/CVF Conference on Computer Vision and Pattern Recognition (CVPR)},
  pages={18902--18913},
  year={2025},
  organization={IEEE}
}

@inproceedings{xiao2024efficient,
  title={Efficient streaming language models with attention sinks},
  author={Xiao, Guangxuan and Tian, Yuandong and Chen, Beidi and Han, Song and Lewis, Mike},
  booktitle={International Conference on Learning Representations},
  volume={2024},
  pages={21875--21895},
  year={2024}
}

@article{han2024free,
  title={Free video-llm: Prompt-guided visual perception for efficient training-free video llms},
  author={Han, Kai and Guo, Jianyuan and Tang, Yehui and He, Wei and Wu, Enhua and Wang, Yunhe},
  journal={arXiv preprint arXiv:2410.10441},
  year={2024}
}

@inproceedings{wang2024videoagent,
  title={Videoagent: Long-form video understanding with large language model as agent},
  author={Wang, Xiaohan and Zhang, Yuhui and Zohar, Orr and Yeung-Levy, Serena},
  booktitle={European Conference on Computer Vision},
  pages={58--76},
  year={2024},
  organization={Springer}
}

@inproceedings{zhong2024memorybank,
  title={Memorybank: Enhancing large language models with long-term memory},
  author={Zhong, Wanjun and Guo, Lianghong and Gao, Qiqi and Ye, He and Wang, Yanlin},
  booktitle={Proceedings of the AAAI conference on artificial intelligence},
  volume={38},
  number={17},
  pages={19724--19731},
  year={2024}
}

@article{wu2022memorizing,
  title={Memorizing transformers},
  author={Wu, Yuhuai and Rabe, Markus N and Hutchins, DeLesley and Szegedy, Christian},
  journal={arXiv preprint arXiv:2203.08913},
  year={2022}
}

@inproceedings{borgeaud2022improving,
  title={Improving language models by retrieving from trillions of tokens},
  author={Borgeaud, Sebastian and Mensch, Arthur and Hoffmann, Jordan and Cai, Trevor and Rutherford, Eliza and Millican, Katie and Van Den Driessche, George Bm and Lespiau, Jean-Baptiste and Damoc, Bogdan and Clark, Aidan and others},
  booktitle={International conference on machine learning},
  pages={2206--2240},
  year={2022},
  organization={PMLR}
}

@inproceedings{you2024ferret,
  title={Ferret: Refer and ground anything anywhere at any granularity},
  author={You, Haoxuan and Zhang, Haotian and Gan, Zhe and Du, Xianzhi and Zhang, Bowen and Wang, Zirui and Cao, Liangliang and Chang, Shih-Fu and Yang, Yinfei},
  booktitle={International Conference on Learning Representations},
  volume={2024},
  pages={57153--57180},
  year={2024}
}

@inproceedings{zhang2023video,
  title={Video-llama: An instruction-tuned audio-visual language model for video understanding},
  author={Zhang, Hang and Li, Xin and Bing, Lidong},
  booktitle={Proceedings of the 2023 conference on empirical methods in natural language processing: system demonstrations},
  pages={543--553},
  year={2023}
}

@inproceedings{lei2020tvqa+,
  title={Tvqa+: Spatio-temporal grounding for video question answering},
  author={Lei, Jie and Yu, Licheng and Berg, Tamara and Bansal, Mohit},
  booktitle={Proceedings of the 58th annual meeting of the association for computational linguistics},
  pages={8211--8225},
  year={2020}
}

@inproceedings{radford2021learning,
  title={Learning transferable visual models from natural language supervision},
  author={Radford, Alec and Kim, Jong Wook and Hallacy, Chris and Ramesh, Aditya and Goh, Gabriel and Agarwal, Sandhini and Sastry, Girish and Askell, Amanda and Mishkin, Pamela and Clark, Jack and others},
  booktitle={International conference on machine learning},
  pages={8748--8763},
  year={2021},
  organization={PmLR}
}

@article{alayrac2022flamingo,
  title={Flamingo: a visual language model for few-shot learning},
  author={Alayrac, Jean-Baptiste and Donahue, Jeff and Luc, Pauline and Miech, Antoine and Barr, Iain and Hasson, Yana and Lenc, Karel and Mensch, Arthur and Millican, Katherine and Reynolds, Malcolm and others},
  journal={Advances in neural information processing systems},
  volume={35},
  pages={23716--23736},
  year={2022}
}

@article{achiam2023gpt,
  title={Gpt-4 technical report},
  author={Achiam, Josh and Adler, Steven and Agarwal, Sandhini and Ahmad, Lama and Akkaya, Ilge and Aleman, Florencia Leoni and Almeida, Diogo and Altenschmidt, Janko and Altman, Sam and Anadkat, Shyamal and others},
  journal={arXiv preprint arXiv:2303.08774},
  year={2023}
}

@article{team2024gemini,
  title={Gemini 1.5: Unlocking multimodal understanding across millions of tokens of context},
  author={Team, Gemini and Georgiev, Petko and Lei, Ving Ian and Burnell, Ryan and Bai, Libin and Gulati, Anmol and Tanzer, Garrett and Vincent, Damien and Pan, Zhufeng and Wang, Shibo and others},
  journal={arXiv preprint arXiv:2403.05530},
  year={2024}
}

@article{glm2024chatglm,
  title={Chatglm: A family of large language models from glm-130b to glm-4 all tools},
  author={Glm, Team and Zeng, Aohan and Xu, Bin and Wang, Bowen and Zhang, Chenhui and Yin, Da and Zhang, Dan and Rojas, Diego and Feng, Guanyu and Zhao, Hanlin and others},
  journal={arXiv preprint arXiv:2406.12793},
  year={2024}
}

@inproceedings{wang2026mmduet2,
  title={Mmduet2: Enhancing proactive interaction of video mllms with multi-turn reinforcement learning},
  author={Wang, Yueqian and Liu, Songxiang and Wang, Disong and Xu, Nuo and Wan, Guanglu and Zhang, Huishuai and Zhao, Dongyan},
  booktitle={International Conference on Learning Representations},
  volume={2026},
  pages={23257--23270},
  year={2026}
}

@article{packer2023memgpt,
  title={Memgpt: Towards llms as operating systems},
  author={Packer, Charles and Wooders, Sarah and Lin, Kevin and Fang, Vivian and Patil, Shishir G and Stoica, Ion and Gonzalez, Joseph E},
  journal={arXiv preprint arXiv:2310.08560},
  year={2023}
}

@inproceedings{park2023generative,
  title={Generative agents: Interactive simulacra of human behavior},
  author={Park, Joon Sung and O'Brien, Joseph and Cai, Carrie Jun and Morris, Meredith Ringel and Liang, Percy and Bernstein, Michael S},
  booktitle={Proceedings of the 36th annual acm symposium on user interface software and technology},
  pages={1--22},
  year={2023}
}

@article{sumers2023cognitive,
  title={Cognitive architectures for language agents},
  author={Sumers, Theodore R and Yao, Shunyu and Narasimhan, Karthik and Griffiths, Thomas L},
  journal={arXiv preprint arXiv:2309.02427},
  year={2023}
}

@misc{qwen2025qwen25vl,
  author       = {{Qwen Team}},
  title        = {Qwen2.5-VL},
  year         = {2025},
  howpublished = {\url{https://qwenlm.github.io/blog/qwen2.5-vl}},
  note         = {Accessed: 2026-09-25}
}

@article{chen2024internvl25,
  title   = {Expanding Performance Boundaries of Open-Source Multimodal Models with Model, Data, and Test-Time Scaling},
  author  = {Chen, Zhe and Wang, Weiyun and Cao, Yue and Liu, Yangzhou and Gao, Zhangwei and Cui, Erfei and Zhu, Jinguo and Ye, Shenglong and Tian, Hao and Liu, Zhaoyang and others},
  journal = {arXiv preprint arXiv:2412.05271},
  year    = {2024}
}
\bibliographystyle{iclr2027_conference}

\newpage
\appendix

\section{Related Works}
\subsection{Proactive Streaming Systems}
\label{app:extended-proactive-streaming}
Passive video question answering (Video-QA) represents a mature domain within multimodal understanding, characterized by retrospective reasoning over fully offline, temporally bounded video sequences \citep{lei2020tvqa+, zhang2023video}. To transition multimodal understanding into interactive contexts, early streaming paradigms introduced online streaming architectures. This includes benchmarks such as StreamingBench \citep{lin2026streamingbench} and OVOBench \citep{niu2025ovo}, which emphasize high-frequency frame-by-frame multimodal inference. However, these systems remain fundamentally question-driven, requiring explicit external queries to trigger each model response and ignoring the problem of autonomous engagement. To formalize autonomous interaction, the paradigm of Proactive Streaming Video QA was introduced by \citet{wang2025proactivevideoqa}, which requires an agent to process an unbounded, continuous streaming input and independently make three interconnected causal decisions at each step: when to speak (timing control), what to say (contextual generation), and whether to deliver the response (emission necessity).

Research addressing this proactive control problem primarily bifurcates into two paradigms, namely parametric trained-head systems and training-free heuristic adaptations. Under the parametric paradigm, architectures such as MMDuet \citep{wang2025videollm}, VideoLLM-Online \citep{chen2024videollm}, and MMDuet2 with Reinforcement Learning \citep{wang2026mmduet2} append dedicated classifier heads to the hidden states of visual backbones, supervising them using dense, frame-level alignment annotations. While displaying high temporal precision on specific, in-distribution scenarios, these models are notoriously sensitive to domain shifts and suffer from catastrophic generalization failure on unseen video densities due to over-fitting on trained temporal priors. Under the training-free paradigm, offline VLMs \citep{li2024llava} are reused through temporal chunking \citep{zhang2024long, chen2025longvila} or heuristic agent-based prompts \citep{han2024free, wang2024videoagent}. Although these systems avoid expensive training, they lack structural mechanisms to explicitly control timing and redundancy, often falling back on brute-force, high-frequency inference that degrades real-time throughput. A parallel line of long-context streaming work addresses engineering challenges via frame-level KV cache compression \citep{xu2026streamingvlm} or attention sinks \citep{xiao2024efficient, xu2026streamingvlm}, yet predominantly assumes passive interaction without active decision gating.

\subsection{Memory-Augmented and Prompt-Based Architecture}
\label{app:extended-memory-prompt}
Modulating frozen vision-language models at inference time via strategic prompt design is an effective, training-free engineering technique. Grounded detail and specification prompts \citep{you2024ferret} forced models to anchor responses in specific spatiotemporal coordinates, while persona prompts \citep{maaz2024video} shaped the stylistic and lexical surface form of outputs. Despite their flexibility, pure prompting methods are inherently memoryless, meaning they cannot dynamically adjust the model’s internal context as the interaction history scales. 
Conversely, memory-augmented architectures employ external storage to modulate runtime generation.

Within these memory-augmented paradigms, Retrieval-Augmented Generation (RAG) methods like RETRO \citep{borgeaud2022improving} and Memorizing Transformers \citep{wu2022memorizing} use dense retrieval networks to query massive external text corpora at each decoding step, which remains computationally expensive and incompatible with real-time stream causality. Database-driven agent memories, such as MemoryBank \citep{zhong2024memorybank} and MemGPT \citep{packer2023memgpt}, compress multi-turn dialogues or interaction summaries into vector databases, executing retrieval queries to fetch historical contexts. However, executing dense database lookups at every frame boundary introduces substantial latency, violating the causality constraints of real-time streaming video understanding. State-driven cognitive architectures, including Generative Agents \citep{park2023generative} and cognitive agent frameworks \citep{sumers2023cognitive}, maintain high-level textual states to plan long-term decisions, but have historically been pursued in isolation from prompt engineering. ProV-Mem addresses these limitations by introducing a unified, non-parametric formulation that co-designs inference-time prompting with explicit state tracking, delivering a lightweight, backbone-agnostic wrapper.

\section{Supplementary Material of Method}
\label{app:supplementary}

\subsection{Mathematical Proof of Lemma~\ref{lem:bound}}
\label{app:proof}

We restate Lemma~\ref{lem:bound} for completeness and provide a proof aligned
with the revised stagewise notation in the main text. The comparison is
counterfactual: at each step, the joint and decomposed processes are evaluated
on the same causal state.

\textbf{Lemma~\ref{lem:bound} (Stagewise Approximation Bound).}
Let $U(\delta)=\mathbb{E}[\sum_{t=1}^{T}u(a_t,r_{g(t)}^*)]$ denote the expected
cumulative utility under decision process $\delta$, where
$u(\cdot,\cdot)\in[0,c]$ is a bounded per-step utility. Let
$\delta_{\mathrm{joint}}^*$ denote the optimal process with access to the full
causal state, and let
$\delta_{\mathrm{dec}}^*=
(\delta_{\mathrm{trig}}^*,\delta_{\mathrm{gen}}^*,
\delta_{\mathrm{emit}}^*)$ denote the decomposed process induced by the three
local memory stages. Let $\epsilon_{\mathrm{trig}}$,
$\epsilon_{\mathrm{gen}}$, and $\epsilon_{\mathrm{emit}}$ denote the average
probabilities that the two processes disagree at the trigger, generation, and
emission stages, respectively. Then
\begin{equation}
U(\delta_{\mathrm{joint}}^*)-
U(\delta_{\mathrm{dec}}^*)
\leq
2cT\bigl(
\epsilon_{\mathrm{trig}}+
\epsilon_{\mathrm{gen}}+
\epsilon_{\mathrm{emit}}
\bigr).
\end{equation}

\textbf{Proof.}

\textbf{Step 1. Per-step utility bound.}
Let $U_t(\delta)$ denote the expected utility contribution at step $t$, such
that $U(\delta)=\sum_{t=1}^{T}U_t(\delta)$. Since
$u(a_t,\cdot)\in[0,c]$, the per-step utility gap is bounded by the difference
between the action distributions induced by the joint and decomposed processes:
\begin{equation}
\left|
U_t(\delta_{\mathrm{joint}}^*)-
U_t(\delta_{\mathrm{dec}}^*)
\right|
\leq
c\,
\left\|
P_{\delta_{\mathrm{joint}}^*}(\cdot\mid S_{\leq t})-
P_{\delta_{\mathrm{dec}}^*}(\cdot\mid S_{\leq t})
\right\|_1.
\end{equation}

\textbf{Step 2. Latent decision chains.}
The joint process follows the standard chain rule
\begin{equation}
P_{\delta_{\mathrm{joint}}^*}(f_t,\hat r_t,e_t\mid S_{\leq t})
=
P(f_t\mid S_{\leq t})
\cdot
P(\hat r_t\mid f_t,S_{\leq t})
\cdot
P(e_t\mid f_t,\hat r_t,S_{\leq t}),
\end{equation}
whereas the decomposed process is factorized by the local memory states:
\begin{equation}
P_{\delta_{\mathrm{dec}}^*}(f_t,\hat r_t,e_t\mid S_{\leq t})
=
P(f_t\mid\mathcal{M}_{\mathrm{trig},t})
\cdot
P(\hat r_t\mid f_t,\mathcal{M}_{\mathrm{inst}},z_t)
\cdot
P(e_t\mid f_t,\hat r_t,\mathcal{M}_{\mathrm{ep},t}).
\end{equation}

The final action is obtained from the latent chain through the deterministic
mapping $a_t=\hat r_t$ when $f_t=e_t=1$, and $a_t=\varnothing$ otherwise.
Therefore, the difference between the induced final-action distributions cannot
exceed the difference between the corresponding latent decision chains.

\textbf{Step 3. Stagewise disagreement bound.}
Under the common-state coupling, a disagreement in the final action can occur
only when at least one stage decision differs between the joint and decomposed
processes. By the union bound, the probability of a final-action disagreement
at step $t$ is at most the sum of the trigger, generation, and emission
disagreement probabilities. Averaging these probabilities over all evaluation
steps gives
\begin{equation}
\left\|
P_{\delta_{\mathrm{joint}}^*}(\cdot\mid S_{\leq t})-
P_{\delta_{\mathrm{dec}}^*}(\cdot\mid S_{\leq t})
\right\|_1
\leq
2\bigl(
\epsilon_{\mathrm{trig}}+
\epsilon_{\mathrm{gen}}+
\epsilon_{\mathrm{emit}}
\bigr).
\end{equation}

\textbf{Step 4. Summing over time.}
Substituting the stagewise bound into the per-step utility bound yields
\begin{equation}
\left|
U_t(\delta_{\mathrm{joint}}^*)-
U_t(\delta_{\mathrm{dec}}^*)
\right|
\leq
2c\bigl(
\epsilon_{\mathrm{trig}}+
\epsilon_{\mathrm{gen}}+
\epsilon_{\mathrm{emit}}
\bigr).
\end{equation}
Summing over $t=1,\ldots,T$ gives
\begin{equation}
U(\delta_{\mathrm{joint}}^*)-
U(\delta_{\mathrm{dec}}^*)
\leq
2cT\bigl(
\epsilon_{\mathrm{trig}}+
\epsilon_{\mathrm{gen}}+
\epsilon_{\mathrm{emit}}
\bigr),
\end{equation}
which completes the proof.

The surface-semantic collision rate $\rho$ is retained as a diagnostic rather
than a term in the bound. It measures the event that a semantically novel
candidate is suppressed by the episodic gate because of high lexical similarity
to the most recent emitted response.

\subsection{Practical Interpretation of Assumption~\ref{asm:suff}}
\label{app:assumption}

Assumption~\ref{asm:suff} is an idealized reference condition under which each
stage can operate from its designated local representation without substantial
loss relative to the full causal state. Lemma~\ref{lem:bound} characterizes
departures from this condition through stagewise discrepancies.

\textbf{Timing sufficiency.}
The trigger stage relies on the temporal state
$\mathcal{M}_{\mathrm{trig},t}=(p_t,c_{\mathrm{sil},t})$, where $p_t$ encodes
the regular trigger cadence and $c_{\mathrm{sil},t}$ counts consecutive
non-emitting trigger opportunities. This representation may be insufficient
when a salient event occurs between regular trigger opportunities and cannot be
identified from temporal state alone. Such cases contribute to
$\epsilon_{\mathrm{trig}}$ and motivate the silence-recovery mechanism.

\textbf{Content sufficiency.}
The generation stage relies on the fixed instructional prior and the active
window representation $z_t=\phi(V_t)$. This representation may be insufficient
when the appropriate response depends on temporally distant evidence or earlier
responses that are unavailable from the current window. In such cases, the
window-level generator may differ from the full-state response, contributing to
$\epsilon_{\mathrm{gen}}$.

\textbf{Emission sufficiency.}
The emission stage relies on the current candidate and the compact episodic
state $\mathcal{M}_{\mathrm{ep},t}=(\bar r_{\mathrm{last},t},
\tau_{\mathrm{last},t}^{\mathrm{emit}})$. This representation may be
insufficient when a candidate is redundant with an older response not captured
by the most recent prefix, or when a semantically novel candidate is suppressed
because of high lexical similarity. Such cases contribute to
$\epsilon_{\mathrm{emit}}$. The latter failure mode is additionally summarized
by the surface-semantic collision rate $\rho$, which is reported as an
interpretable diagnostic rather than the complete emission-stage discrepancy.

\subsection{Prompt Templates}
\label{app:prompts}

To ensure reproducibility and a clear separation of roles across the three memory tiers, we summarize the prompt templates used in ProV-Mem.

\textbf{Primary System Prompt $P_{\mathrm{sys}}$ (Instructional Memory $\mathcal{M}_{\mathrm{inst}}$).}
This prompt is shared across scenarios as the fixed instructional prior. Its role is to reduce $\epsilon_{\mathrm{gen}}$ and $\rho$ by encouraging concise, canonicalized, and visually grounded responses.

\begin{lstlisting}[style=promptstyle]
You are a real-time visual assistant watching a continuous video stream. Your task is to generate a concise, factual, and objective description of the active event in the current visual window.

To ensure consistency and integration with the downstream real-time system, you MUST strictly adhere to the following constraints:
1. Provide a single self-contained sentence of no more than 15 words.
2. Focus only on visible actions and physical interactions in the current window.
3. Do NOT reference past events, do NOT use introductory phrases (e.g., "In this video", "We can see"), and do NOT speculate about future actions.
4. Use standardized and canonical vocabulary with simple active-voice constructions.
5. If no novel or meaningful event is occurring, or if the scene is inactive, you MUST output exactly: "No Reply".
\end{lstlisting}

\paragraph{Silence-Recovery Contingency Prompt $P_{\mathrm{sys}}^{\mathrm{force}}$ (TV scenario).}
This prompt is used in the TV scenario when the silence counter reaches $\kappa$ and the system must recover from prolonged dormancy.

\begin{lstlisting}[style=promptstyle]
You are a real-time visual assistant. The system has detected a prolonged period of silence and inactivity. You MUST generate a description of the current visual window to restore interaction, even if the event is minor.

Constraints:
1. Generate exactly one short factual sentence.
2. Focus on any visible character, posture, or background cue if no major action is occurring.
3. Do NOT output "No Reply". You must produce a valid description.
\end{lstlisting}

\paragraph{Lightweight Scouting Prompt $P_{\mathrm{scout}}$ (VAD scenario).}
This prompt is used in the VAD scenario as a lightweight high-recall pre-screening stage. Its purpose is to decide whether the current window should be forwarded to the full response generator.

\begin{lstlisting}[style=promptstyle]
You are a high-speed security scouting agent. Analyze the current visual window and determine whether there is any ongoing anomaly, security incident, or unusual activity.

You must output a JSON object with exactly two keys:
{
  "anomaly_detected": true/false,
  "category": "none" or a short standardized event category
}

Only output the raw JSON object. Do not output any other text.
\end{lstlisting}

\section{Experiment}
\subsection{Dataset}
To evaluate autonomous and interactive capabilities under strict temporal causality, recent benchmarks have begun to emphasize proactive streaming paradigms. A notable milestone is ProactiveVideoQA (also formalized as ProactiveStreaming) \citep{wang2025proactivevideoqa}, a large-scale streaming multimodal benchmark specifically designed to assess an agent’s proactive interaction capacity over unbounded visual inputs. To comprehensively evaluate proactive agents across diverse temporal densities and semantic distributions, the benchmark incorporates four distinct streaming tasks:
\begin{itemize}
    \item Streaming Video Captioning (WEB), which evaluates continuous, low-latency web video narration.
    \item Egocentric Video Narration (EGO), which demands high-frequency first-person activity tracking characterized by dense, repetitive action substeps.
    \item Subtitle-centric Dialogue QA (TV), which requires the temporal alignment of visual evidence with streaming conversational subtitle flows.
    \item Video Anomaly Detection (VAD), which assesses the model's ability to trigger timely, non-redundant alarms over highly sparse and critical hazard events.
\end{itemize}

Under this protocol, models must dynamically synchronize trigger control, grounded-content generation, and lexical redundancy filtering under strict online constraints, establishing a challenging and unified testbed for next-generation proactive vision-language agents.
\subsection{Experimental Setting}
\label{app-exp:setting}

We evaluate ProV-Mem on all four scenarios of the ProactiveVideoQA dataset. We report the standard Proactive Area Under the Curve (PAUC) at trade-off weight $\omega = 0.5$ together with our proposed Cost-Aware Area Under the Curve (CAUC). Both metrics are computed with the official evaluation framework under the official single-question protocol, using GPT-4.1 as the semantic judge. Unless otherwise noted, all experiments use the same frozen Qwen2.5-VL-7B-Instruct backbone.
The main domain-specific differences lie in trigger cadence, episodic control, and a small number of scenario-specific prompting modules. 
For WEB, we use a regular trigger stride of $K=3$ sampled frames, a similarity threshold of $\eta = 0.75$, and no suppression expiry in episodic memory, that is, $\tau_{\mathrm{sup}} = \infty$. EGO uses the same trigger and similarity settings but introduces a suppression cooldown of $\tau_{\mathrm{sup}} = 30$\,s to permit re-emission of recurring procedural steps after sufficiently long gaps. TV also uses $K=3$ and $\eta = 0.75$, and activates a silence-recovery prompt $P_{\mathrm{sys}}^{\mathrm{force}}$ after $\kappa = 4$ consecutive silent trigger windows. VAD adopts a two-stage pipeline with a lightweight scouting prompt $P_{\mathrm{scout}}$ and an event cooldown of $\tau_{\mathrm{event}} = 8$\,s. For CAUC, we fix the penalty weights to $\alpha = 0.3$ for redundancy, $\beta = 0.3$ for frequency, and $\gamma = 0.4$ for token inefficiency. We set the length reference to $L_{\mathrm{ref}} = 50$ tokens and the cooldown to $\tau_{\mathrm{cooldown}} = 5$\,s for all four scenarios.
We set the comparison prefix length to $L=200$ characters in all experiments.

\subsection{Per-Domain Implementation Details}

For WEB, frames are sampled at 1 fps and grouped every $K=3$ frames. At each trigger window, the backbone receives the current frame batch together with the primary system prompt $P_{\mathrm{sys}}$ and generates a reply of up to 128 tokens using greedy decoding. The episodic memory $\mathcal{M}_{\mathrm{ep}}$ stores a fixed length prefix extracted from the last emitted reply. A newly generated candidate is suppressed when its character level lexical similarity to the stored prefix, measured by longest common subsequence alignment, exceeds $\eta = 0.75$. 
Since $\tau_{\mathrm{sup}} = \infty$, the suppression state is not reset within the video. A negation filter additionally maps outputs containing phrases such as ``there is no'', ``cannot see'', or ``not visible'' to \texttt{No Reply}. The instructional memory $\mathcal{M}_{\mathrm{inst}}$ keeps the fixed prompt $P_{\mathrm{sys}}$ throughout the video, with no per-video adaptation.

For EGO, the inference pipeline is identical to WEB, including 1 fps frame sampling, $K=3$, and $\eta = 0.75$, except that a suppression cooldown of $\tau_{\mathrm{sup}} = 30$\,s is introduced in $\mathcal{M}_{\mathrm{ep}}$. This allows the system to emit again for the same procedural substep after a sufficiently long interval, which is necessary because egocentric cooking and assembly videos often revisit the same action after a temporal gap.

For TV, the model processes the frame stream together with the most recent $W=6$ unique subtitle lines, which are provided in each trigger window as a \texttt{[Recent subtitles]} block. The trigger cadence is fixed at $K=3$ frames, and the episodic suppression threshold remains $\eta = 0.75$. When $\kappa = 4$ consecutive trigger windows pass without an emitted response, corresponding to 12 seconds under 1 fps sampling, the system invokes the recovery prompt $P_{\mathrm{sys}}^{\mathrm{force}}$ to produce a best effort reply, after which the silence counter is reset. In addition, if no valid reply has been emitted throughout the clip, the system performs a single end of video fallback. This mechanism does not require access to gold timestamps at inference time.

VAD employs a two-stage pipeline. In the scouting stage, every $K=3$ frames are processed by the lightweight JSON-output prompt $P_{\mathrm{scout}}$ in a high-recall mode. The backbone returns a binary anomaly flag together with a coarse event category drawn from the UCF-Crime anomaly label space. Any window predicted as \texttt{anomaly\_detected=true}, or matched by a keyword-based fallback rule, is forwarded to the reply stage. In the reply stage, the backbone generates the final natural-language description using $P_{\mathrm{sys}}$. The episodic module then applies the same LCS-based character-level similarity gate with $\eta = 0.75$ to suppress near-duplicate responses, while an event cooldown of $\tau_{\mathrm{event}} = 8$\,s prevents repeated emissions for the same anomaly. A five-frame normal baseline at the start of each video is used to calibrate the scouting stage before anomaly detection begins.

\subsection{Sensitivity to the Timeliness/Correctness Trade-off}
\begin{figure}[!t]
  \centering
  \includegraphics[width=\linewidth]{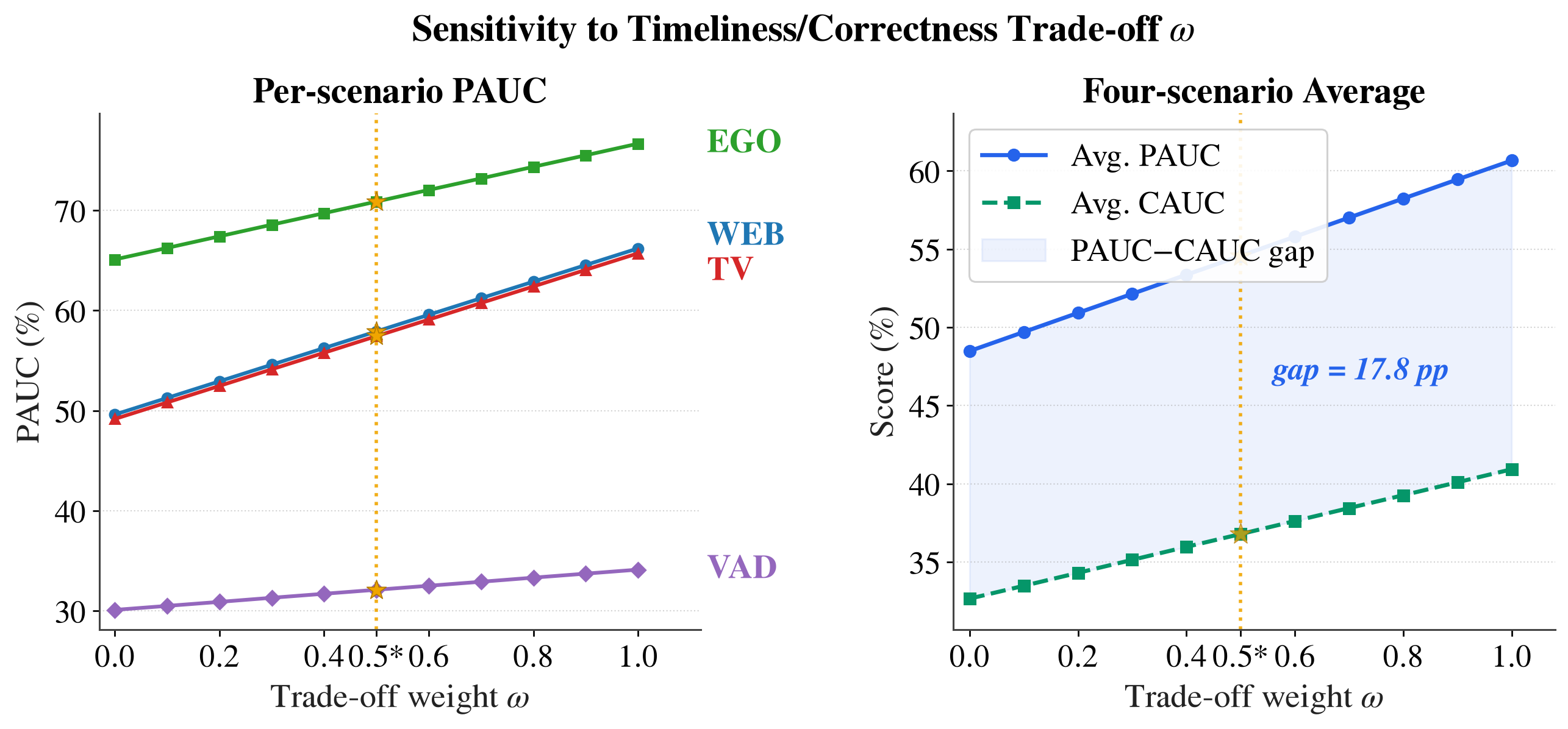}
  \caption{Sensitivity of ProV-Mem to the timeliness/correctness trade-off. Left: per-scenario PAUC across $\omega \in [0,1]$. Right: four-scenario average PAUC and CAUC with the PAUC–CAUC gap shaded.}
  \label{fig:3}
\vspace{-0.6em}
\label{fig:h3}
\end{figure}

To verify that our results are not tied to the specific choice of $\omega = 0.5$, we re-score the same emitted transcripts over $\omega \in \{0.0, \ldots, 1.0\}$. As expected from the metric definition, both PAUC and CAUC vary smoothly and nearly linearly with $\omega$ across all scenarios, with no anomalous reversals or discontinuities. More importantly, ProV-Mem remains strong even under the most timing-sensitive setting ($\omega = 0$). For example, EGO still achieves a PAUC of 65.09, indicating that the gain is not merely a consequence of a favorable trade-off weight or loose timing tolerance. Instead, the improvement persists when timeliness is emphasized most strongly, which is consistent with the claim that response quality remains a major contributor to overall performance.


\subsection{Sensitivity to Suppression Cooldown and Silence-Recovery  Threshold}

\begin{figure}[!t]
  \centering
  \includegraphics[width=0.92\linewidth]{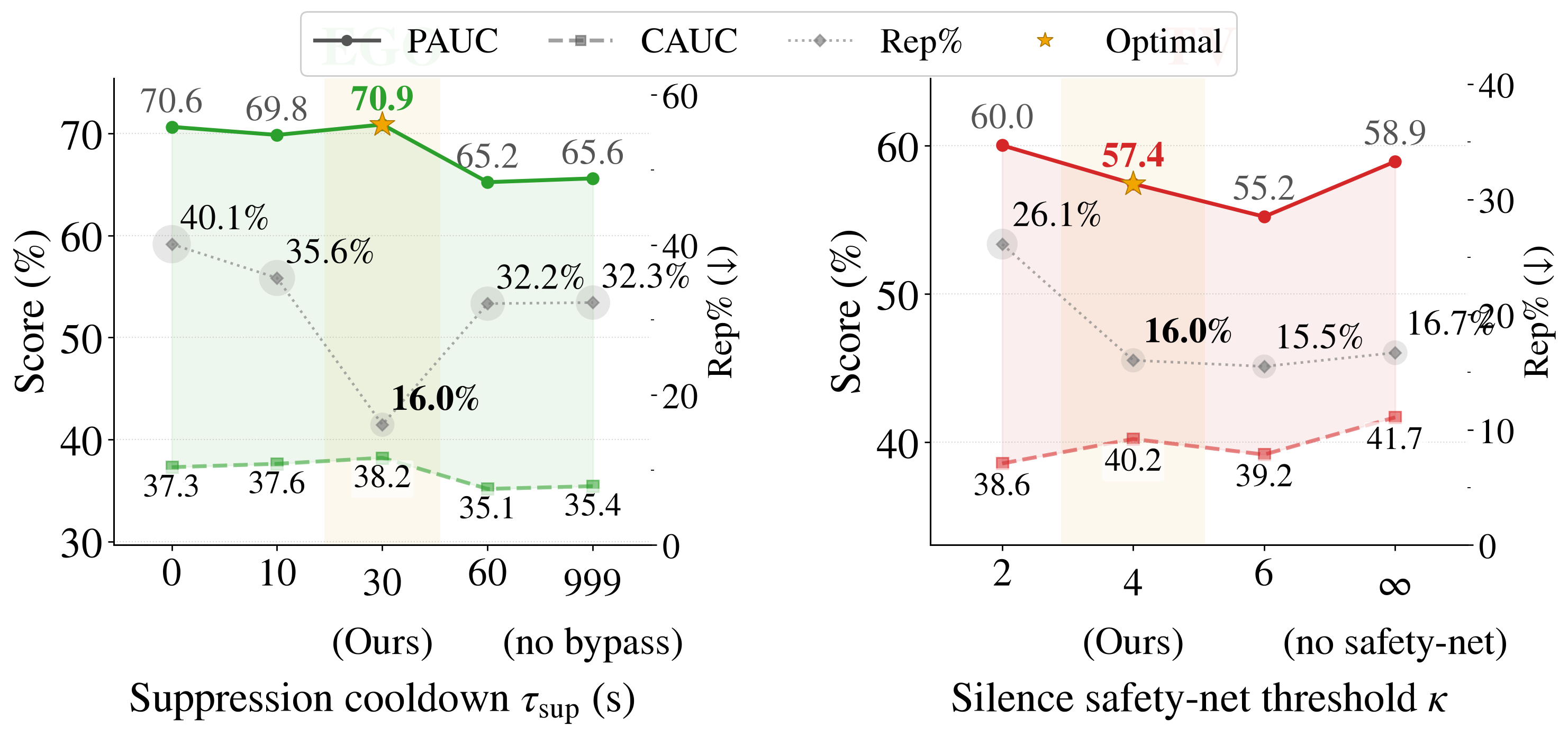}
  \caption{Left: effect of the suppression cooldown $\tau_{\mathrm{sup}}$ on EGO. Right: effect of the silence-recovery threshold $\kappa$ on TV. Highlighted columns are the defaults used throughout.}
  \label{fig:h3h4}
\vspace{-0.6em}
\end{figure}

\paragraph{Suppression cooldown $\tau_{\mathrm{sup}}$ on EGO.}
The suppression cooldown $\tau_{\mathrm{sup}}$ in Eq.~\ref{eq:emit} controls how long an accepted reply continues to suppress lexically similar candidates. Once $t - t_{\mathrm{last}} > \tau_{\mathrm{sup}}$, the episodic state is effectively reset, allowing a similar candidate to be emitted again. With $\tau_{\mathrm{sup}} = 0$, ProV-Mem attains a PAUC of 70.64 but incurs a Rep\% of 40.11\%, indicating that near-duplicate procedural descriptions are repeatedly emitted across adjacent windows. Setting $\tau_{\mathrm{sup}} = 30$\,s yields the strongest overall operating point in this sweep, with PAUC/CAUC of 70.87/38.20 and a much lower Rep\% of 16.04\%. This value closely matches the recurrence scale of egocentric substeps, whose median inter-GT gap is about 28.7\,s. Larger values reduce PAUC to 65.2 to 65.6, suggesting that valid re-occurrences are then suppressed for too long.

\paragraph{Silence silence-recovery threshold $\kappa$ on TV.}
The silence-recovery threshold $\kappa$ specifies after
how many consecutive silent regular trigger opportunities an additional recovery
query is issued, preventing the model from remaining dormant on sparse streams. A small threshold such as $\kappa = 2$ makes this recovery mechanism too aggressive: PAUC rises to 60.02, but Rep\% also increases to 26.13\%, because forced triggers are frequently invoked in weakly informative windows. The default setting $\kappa = 4$ yields a PAUC of 57.43 and a CAUC of 40.22. Under 1 fps sampling, this corresponds to 12\,s of silence and serves as a moderate operating point that preserves a recovery path for sparse events without making forced triggering overly frequent. At the other extreme, disabling the silence-recovery  entirely ($\kappa = \infty$) gives PAUC 58.93, CAUC 41.68, and Rep\% 16.70\%, which indicates lower redundancy but also removes the recovery mechanism for sparse events. Overall, the sweep shows that $\kappa$ mainly controls the balance between recovery coverage and redundant activation, rather than uniformly improving all metrics at once.

\end{document}